\documentclass[twoside,11pt]{article}
\usepackage{amsmath,amssymb,amsthm}
\usepackage{jmlr2e}
\usepackage{blindtext}

\let\jmlrproof\proof
\let\endjmlrproof\endproof

\let\proof\jmlrproof
\let\endproof\endjmlrproof

\makeatletter
        
\let\c@theorem\relax      
         
\let\c@remark\relax       
    
\let\c@proposition\relax  
\makeatother

\theoremstyle{plain}
\newtheorem{theorem}{Theorem}
\newtheorem{proposition}{Proposition}
\theoremstyle{remark}
\newtheorem{remark}{Remark}

\usepackage{lastpage}
\jmlrheading{2X}{202X}{1-\pageref{LastPage}}{X/25; Revised 5/2X}{9/2X}{21-0000}{Min Gan and Author Two}

\ShortHeadings{A Saddle Point Remedy: Power of Variable Elimination}{Min Gan}
\firstpageno{1}

\begin{document}

\title{A Saddle Point Remedy: Power of Variable Elimination in Non-convex Optimization}

\author{\name Min Gan \email aganmin@gmail.com \\
       \addr College of Computer Science and Technology\\
       Qingdao University\\
       Qingdao 266071, China
       \AND
       \name Guang-yong Chen \email cgykeda@mail.ustc.edu.cn\\
       \addr College of Computer and Data Science\\
       Fuzhou University\\
       Fuzhou 350116, China
       \AND
       \name Yang Yi \email yiyang@yzu.edu.cn \\
       \addr College of Information Engineering\\
       Yangzhou University\\
       Yangzhou, 225000, Jiangsu, China
       \AND
       \name Jing Chen \email 8201703038@jiangnan.edu.cn\\
       \addr School of Science\\
       Jiangnan University\\
       Wuxi, 214000, Jiangsu, China
       \AND
       \name Lin Wang \email wangplanet@gmail.com\\
       \addr Shandong Provincial Key Laboratory of Network Based Intelligent Computing\\
       University of Jinan \\
       Jinan 250022, China
       }

\editor{My editor}

\maketitle

\begin{abstract}
The proliferation of saddle points, rather than poor local minima, is increasingly understood to be a primary obstacle in large-scale non-convex optimization for machine learning. Variable elimination algorithms, like Variable Projection (VarPro), have long been observed to exhibit superior convergence and robustness in practice, yet a principled understanding of why they so effectively navigate these complex energy landscapes has remained elusive. In this work, we provide a rigorous geometric explanation by comparing the optimization landscapes of the original and reduced formulations.  Through a rigorous analysis based on Hessian inertia and the Schur complement, we prove that variable elimination fundamentally reshapes the critical point structure of the objective function, revealing that local maxima in the reduced landscape are created from, and correspond directly to, saddle points in the original formulation. Our findings are illustrated on the canonical problem of non-convex matrix factorization, visualized directly on two-parameter neural networks, and finally validated in training deep Residual Networks, where our approach yields dramatic improvements in stability and convergence to superior minima. This work goes beyond explaining an existing method; it establishes landscape simplification via saddle point transformation as a powerful principle that can guide the design of a new generation of more robust and efficient optimization algorithms.
\end{abstract}


\section{Introduction}
Machine learning algorithms train their models and perform inference by solving optimization problems. Many of these optimization problems have a separable structure, wherein the decision variables can be partitioned into distinct groups--an insight that can be leveraged to enhance algorithmic efficiency. A common special case is when one set of variables enters the objective in a simple (often linear or quadratic) way, while the remaining variables enter nonlinearly. In these scenarios, it is often possible to eliminate the simple variables by solving for their optimal values in closed form as functions of the other variables. This technique goes by various names, including variable projection \citep{golub:03, newman:21}, Wiberg algorithm \citep{okatani:07, strelow:15}, partial minimization \citep{aravkin:17}, or taking a Schur complement of the problem \citep{demmel:21, weber:23}. The central idea is to reduce the original problem in \emph{both dimension and complexity} by analytically optimizing over one block of variables, yielding a smaller optimization problem in the remaining variables only.

 Let us begin by formally defining the class of separable optimization problems under consideration, followed by a simple but illustrative example. Specifically, we focus on unconstrained optimization problems of the form

\begin{equation}
    \min_{\boldsymbol{\theta} \in \mathbb{R}^n} f(\boldsymbol{\theta}),
\end{equation}
where $f: \Re^n \to \Re$ is the objective function, and the decision variable \(\boldsymbol{\theta}\) can be naturally partitioned as
\begin{equation}
    \boldsymbol{\theta} = \begin{pmatrix}
        \boldsymbol{\alpha} \\
        \boldsymbol{\beta}
    \end{pmatrix}, \boldsymbol{\alpha} \in \mathbb{R}^p, \boldsymbol{\beta} \in \mathbb{R}^q, p+q = n
\end{equation}
in such a way that solving the subproblem
\begin{equation}
    \min_{\boldsymbol{\beta}} f(\boldsymbol{\alpha}, \boldsymbol{\beta})
\end{equation}
is analytically or numerically straightforward for every fixed $\boldsymbol{\alpha}$ in the domain under consideration.

Let ${\boldsymbol\beta}(\boldsymbol{\alpha})$ denote a solution to the subproblem (3). By substituting back into the original objective function, we can reformulate the problem as
\begin{equation}
    \min_{\boldsymbol{\alpha} \in \mathbb{R}^q} \varphi(\boldsymbol{\alpha}),
\end{equation}
where the reduced objective function $\varphi(\boldsymbol{\alpha})$ is defined as
\begin{equation}\nonumber
\varphi(\boldsymbol{\alpha}) = f(\boldsymbol{\alpha},\boldsymbol{\beta}(\boldsymbol{\alpha})).
\end{equation} 

This reformulation effectively eliminates the variable \({\boldsymbol\beta}\),  reducing the original $n$-dimensional optimization problem to a $p$-dimensional one, at the cost of solving a $q$ dimensional minimization problem (3). In the following, we refer to algorithms that simultaneously optimize all parameters in (1) without considering the problem's structure as the joint method, and algorithms applied to (4) as the variable elimination method.

We illustrate the variable elimination approach using the two-dimensional Rosenbrock function [\cite{wiki:rosenbrock}], defined as
\begin{equation}
    f(x,y) = (1-x)^2 + 100(y-x^2)^2.
\end{equation}

This function is well known for its narrow, curved valley leading to the global minimum at $(x,y) = (1,1)$. While it is straightforward to locate the valley, convergence to the global minimum along this valley is notoriously slow due to the flat curvature in the $x$-direction and steep curvature in the $y$-direction. However, for any fixed $x$, the optimal $y$ is exactly $x^2$, which allows us to eliminate $y$ and formulate the reduced objective function as 
\begin{equation}
    \varphi(x) = (1-x)^2.
\end{equation}

\begin{figure}[!t]
    \centering
    \includegraphics[height=0.5\textwidth]{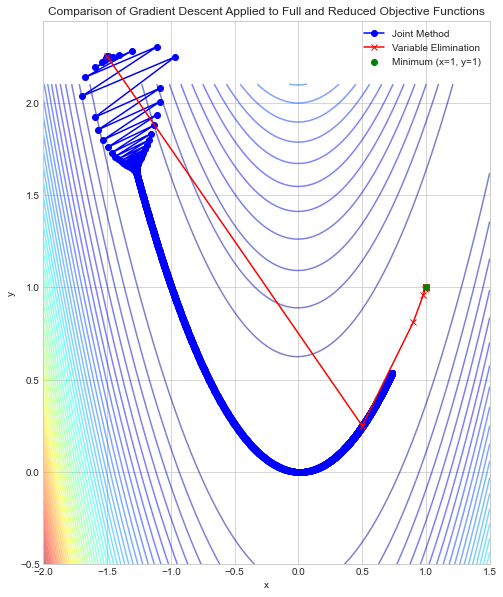}
    \caption{Optimization paths of gradient descent applied to the original Rosenbrock function and the reduced function from the same initial point (-1.5, 2.25). The blue path represents the conventional approach, which exhibits certain oscillations and twists at the beginning and then creeps along the flat valley. The red path represents the variable elimination approach, which quickly reaches the minimum in only a few steps.}
\end{figure}

Fig. 1 shows the trajectories of the same vanilla gradient descent (VGD) algorithm applied to the full objective function (5) and the reduced one (6). In the figure, the blue path shows that VGD applied to the full objective function (5) makes many tiny steps along the curve toward the minimum, wasteing iterations "zigzagging" or creeping along the flat valley. For a stable step size (about $10^{-3}$ here), it required thousands of iterations to approach the minimum. Any larger step size causes the method to overshoot or even diverge due to the valley's steep sides. In the variable elimination approach, we analytically eliminate $y$ by setting $y=x^2$. This reduces the problem to (6), a simple one-dimensional quadratic. The red path in the figure shows that the algorithm stays exactly on the valley floor at all times. In our example, a step size of $\alpha=0.4$ moves $x$ from $-1.5$ all the way to $0.5$ in one step (red square in the figure), and reaches near $x=1$ in only a few iterations. Because the reduced function $(x-1)^2$ is well-behaved, the variable elimination approach converges to $(1,1)$ in just a handful of iterations. This method is dramatically faster in this case, since it never has to "feel out" the correct $y$ --- it always stays in the optimal subspace for $y$. Another important difference between the unseparated and separated approach is robustness to step size. Standard gradient descent on Rosenbrock's function is extremely sensitive to the choice of step size (learning rate). The Rosenbrock valley has steep curvature in the $y$ direction and shallow curvature along $x$, leading to a very high condition number for the Hessian. By contrast, the variable elimination reduces the problem to one dimension with curvature $\varphi''(x)=2$ everywhere (condition number $1$). This well-conditioned problem tolerates much larger steps (up to $\alpha<1$ for convergence in theory).

\begin{figure}[!t]
    \centering
    \includegraphics[height=0.35\textwidth]{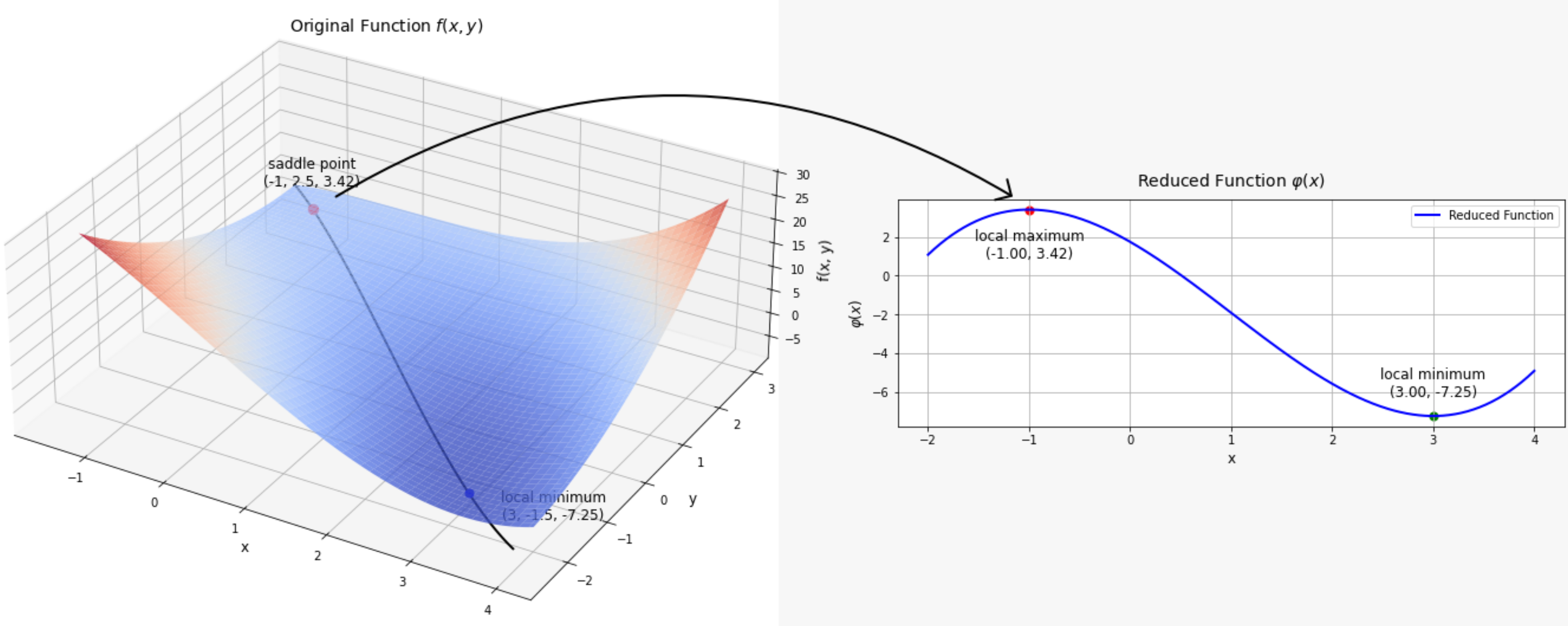}
    \caption{Variable elimination turns a saddle point to a local maximum. Left: the original function $f(x,y)=\frac{1}{3}x^3 + y^2 + 2xy - 6x - 3y + 4$ which has a saddle point and a local minimum. Right: the reduced function $\varphi(x) = \frac{1}{3}x^3 - x^2 - 3x + \frac{7}{4}$ which has a local maximum and a local minimum. The variable elimination approach search along the black line in the original space and is reflect in the left figure by blue line.}
\end{figure}

The Rosenbrock example clearly illustrates the practical advantages of the variable elimination approach, but it is important to recognize that such dramatic success should not be expected in all cases, as real-world problems are often far more complex. While the dramatic speedup on the Rosenbrock function is illustrative, the benefits of variable elimination extend far beyond well-conditioned subproblems. A wealth of numerical experiments has shown that optimization algorithms applied to the reduced problem not only converge substantially faster but also frequently find solutions of superior quality. This empirical success raises a deeper, more fundamental question that has largely remained unanswered:\\

 \textit{Why are variable elimination algorithms more likely to find better minima compared to methods that optimize the full-dimensional problem? }\\
 
 This question serves as the core motivation for this research, as we aim to uncover the underlying mechanisms and conditions that allow variable elimination to outperform conventional approaches in complex, non-convex optimization tasks. To build intuition towards an answer, let us consider a second example that directly probes the geometry of the optimization landscape. Consider the illustrative function
\begin{equation}
   f(x,y)=\frac{1}{3}x^3 + y^2 + 2xy - 6x - 3y + 4.
\end{equation}

A detailed graphical and analytical examination (see Fig. 2) reveals a notable saddle point at $(-1, 2.5)$ and a local minimum at $(3, -1.5)$. By employing the variable elimination approach - specifically substituting the constraint $y = \frac{3-2x}{2}$ into the original function - we obtain a reduced one-dimensional function
\begin{equation}
    \varphi(x) = \frac{1}{3}x^3 - x^2 - 3x + \frac{7}{4}.
\end{equation}

Remarkably, under this dimension reduction, the saddle point transforms into a local maximum. This example succinctly demonstrates the advantage of the variable elimination method: it reshapes the optimization landscape, effectively removing saddle points and converting them into local extrema that can be more reliably identified and navigated by optimization algorithms. Such insight underscores the potential of variable elimination strategies in simplifying complex optimization landscapes, enhancing both the convergence properties and robustness of optimization algorithms.

Throughout this paper a lower-case letter in boldface will indicate a column vector, while the same letter with a subscript will indicate a component of the vector. Upper case stands for matrices. The transpose of the matrix $\boldsymbol{A}$ is denoted by $\boldsymbol{A}^T$. We use the Euclidean vector norm
\begin{equation}\nonumber
\|\boldsymbol{x}\|_2 = (\boldsymbol{x}^T \boldsymbol{x})^\frac{1}{2}=(\Sigma x_i^2)^{\frac{1}{2}}.
\end{equation} 

\section{Related Work and Analysis}
 Variable elimination techniques have their roots in the seminal work by \cite{golub:73} in the 1970s. They studied a class of separable nonlinear least squares (SNLLS) problems and introduced the \textit{variable projection} (VarPro) method, leveraging the differentiation of pseudoinverses to eliminate linear variables analytically. The SNLLS problem is prevalent across various applications in different fields, as many inverse problems can be interpreted as nonlinear data fitting tasks.  In such problems, the residual vector depends linearly on a subset of parameters and nonlinearly on the others. Its general form can be expressed as follows:
\begin{equation}
    \min_{\boldsymbol{x} \in \mathbb{R}^p, \boldsymbol{y} \in \mathbb{R}^q} \frac{1}{2}\|\boldsymbol\epsilon(\boldsymbol{x}, \boldsymbol{y})\|_2^2 = \frac{1}{2}\|\boldsymbol G(\boldsymbol{x})\boldsymbol{y} - \boldsymbol z\|_2^2
\end{equation}
where $\boldsymbol{x}$ denotes the nonlinear parameters, $\boldsymbol{y}$ the linear parameters, 
$\boldsymbol{G}(\boldsymbol{x}) \in \mathbb{R}^{m \times q}$ a matrix-valued function of $\boldsymbol{x}$, and $\boldsymbol{z} \in \mathbb{R}^m$ the observed data. The key insight of VarPro is to eliminate $\boldsymbol{y}$ by solving the linear least squares subproblem:
\begin{equation}
  \boldsymbol{y}^\star(\boldsymbol{x}) = \arg\min_{\boldsymbol{y}} \|\boldsymbol{G}(\boldsymbol{x})\boldsymbol{y} - \boldsymbol{z}\|_2^2,
\end{equation}
which has the closed-form solution $\boldsymbol{y}^\star(\boldsymbol{x}) = \boldsymbol{G}(\boldsymbol{x})^{\dagger} \boldsymbol{z}$, where $\boldsymbol{G}(\boldsymbol{x})^{\dagger}$ denotes the Moore-Penrose pseudoinverse. Substituting this back into the original objective yields the \textit{reduced function}:
\begin{equation}
    \varphi(\boldsymbol{x}) = \frac{1}{2} \| \boldsymbol{G}(\boldsymbol{x}) \boldsymbol{G}(\boldsymbol{x})^{\dagger} \boldsymbol{z} - \boldsymbol{z} \|_2^2,
\end{equation}
which depends only on the nonlinear parameters $\boldsymbol{x}$. 
 
 Since then, the algorithm has undergone extensive development and refinement \citep{kaufman:75, ruhe:80, mullen:08, chung:10, aravkin:12, oleary:13, shearer:13, aravkin:17, espanol:25, chen:25}. One of the earliest and most influential follow-up studies was conducted by \citet{ruhe:80}, who analyzed the asymptotic convergence of both separated algorithms (VarPro) and unseparated algorithm (where all variables are optimized together in a joint fashion). They concluded that both methods have comparable convergence rates in the asymptotic limit. However, it is important to note the context of \citet{ruhe:80}'s conclusion. Their analysis was confined to the local (asymptotic) regime and considered essentially a single Gauss-Newton step near the optimum. In other words, their study looked at the idealized final behavior of the algorithms, but did not address the global optimization behavior or the landscape properties that dictate performance far from the optimum. It is precisely this global behavior that has driven VarPro's adoption in challenging real-world applications. 
 
 Beyond the core algorithmic advances, the variable projection idea has been widely applied in various domains, with significant success in machine learning and computer vision \citep{sjoberg:97, pereyra:06, okatani:07, okatani:11, eriksson:12, strelow:15, hong:17, gan:18, chen:19, demmel:21, newman:21, chen:21, newman:22, weber:23}. An early example is Wiberg's algorithm \cite{wiberg:76}, which applied VarPro to low-rank matrix factorization (principal component analysis with missing data). Despite its promise, this approach was largely overlooked for decades; many computer vision methods in the 1990s-2000s relied on alternating least-squares or joint optimization, which often suffered from poor convergence behavior. Renewed interest in VarPro arose in the 2000s when researchers recognized its advantages for challenging optimization problems. For instance, \cite{okatani:07} revisited Wiberg's method and demonstrated its strong convergence properties for structure-from-motion and matrix completion tasks, vastly outperforming standard joint optimization methods. Subsequent work by \cite{okatani:11} combined VarPro with damping strategies (in a Levenberg-Marquardt framework) to boost convergence rates to near 100\% on previously difficult factorization problems.
  
 Most recently, \cite{hong:17} revisited the VarPro method, seeking to explain its superior performance on ill-conditioned problems like affine bundle adjustment. They argue that under ill-conditioned circumstances, a joint optimization approach yields a much smaller update to the linear parameters ($\Delta \boldsymbol{y}_{\text{Joint}}$) compared to the update from the variable projection method ($\Delta \boldsymbol{y}_{\text{VarPro}}$). In their paper, they present the VarPro update formula to support this claim:
 \begin{equation}
   \Delta \boldsymbol{y}_{\text{VarPro}} = - \boldsymbol{J}_{\boldsymbol{y}}(\boldsymbol{x}_k)^{\dagger}(\boldsymbol\epsilon(\boldsymbol{x}_k,\boldsymbol{y}_k )+\boldsymbol{J}_{\boldsymbol{x}}(\boldsymbol{x}_k,\boldsymbol{y}_k )\Delta \boldsymbol{x})
 \end{equation}
 where $\boldsymbol{J}_{\boldsymbol{x}}(\boldsymbol{x}, \boldsymbol{y})=\partial \boldsymbol\epsilon(\boldsymbol{x}, \boldsymbol{y})/\partial \boldsymbol{x}$ and $\boldsymbol{J}_{\boldsymbol{y}}(\boldsymbol{x})=\partial \boldsymbol\epsilon(\boldsymbol{x}, \boldsymbol{y})/{\partial \boldsymbol{y}}$  the Jacobian matrices with respect to $\boldsymbol{x}$ and $\boldsymbol{y}$, respectively.
 
 However, we note that their expression is an approximation (see Appendix A). The accurate $\Delta \boldsymbol{y}_{\text{VarPro}}$ can be derived using Newton's method (see Appendix A), which gives 
 \begin{equation}
    \Delta \boldsymbol{y}_{\text{VarPro}} = - \boldsymbol{J}_{\boldsymbol{y}}(\boldsymbol{x}_{k+1})^{\dagger}\boldsymbol\epsilon(\boldsymbol{x}_{k+1},\boldsymbol{y}_k ). 
 \end{equation}
 If the accurate expression is used, the dependency on the updated parameter $\boldsymbol{x}_{k+1}$ fundamentally complicates a direct comparison with $\Delta \boldsymbol{y}_{\text{Joint}}$.
 
 More importantly, since the search paths of the joint and VarPro methods diverge after the first step, the state $(\boldsymbol{x}_k,\boldsymbol{y}_k)$ is different for each algorithm at any given iteration $k$. Therefore, a direct comparison of update magnitudes from these different states does not provide a meaningful basis for judging overall algorithmic performance. This highlights the need for an analysis rooted in the intrinsic properties of the optimization landscapes themselves, rather than in the behavior of specific update steps, which is the central goal of our work.

\section{Main results}
While asymptotic analyses suggest similar local convergence rates, modern research has demonstrated that VarPro and joint optimization can follow vastly different optimization trajectories. These path differences are often critical in determining whether an algorithm finds a high-quality solution, especially on non-convex landscapes. Empirically, VarPro-based methods have exhibited a significantly larger basin of convergence; they often converge to a global or near-global optimum from random initializations where standard joint methods fail. For instance, in bundle adjustment and low-rank matrix factorization problems (which belong to classic SNLLS scenarios), a straightforward joint optimization frequently stalls or converges to a poor point when initiated far from the solution. In contrast, VarPro, which eliminates a subset of variables at each iteration, has been observed to find the best-known optimum from random starts in these same problems. 

This section illuminates the reasons for this performance difference by exploring how the optimization landscape is transformed by variable elimination. Specifically, we analyze the relationship between the stationary points of the original and reduced problems. While some aspects of this relationship for SNLLS were explored by \cite{golub:73}, we extend this analysis to a broader class of generalized separable optimization problems.

\subsection{Foundations: Mapping Stationary Points}
Consider a general separable optimization problem:
\begin{equation}
    \min_{(\boldsymbol{x},\boldsymbol{y}) \in \mathbb{R}^p 
    \times \mathbb{R}^q} F(\boldsymbol{x}, \boldsymbol{y})
\end{equation}
where $F: \mathbb{R}^p \times \mathbb{R}^q \to \mathbb{R}$ is assumed to be a continuously differentiable function. This problem is \textit{separable} in the sense that for any fixed $\boldsymbol{x}$ in the domain under consideration, the minimization subproblem with respect to $\boldsymbol{y}$ is tractable and can be solved efficiently, either analytically or numerically.. This structure is the core of the VarPro methodology. This general form captures SNLLS as a special case where $F(\boldsymbol{x},\boldsymbol{y})$ has a specific structure, but also allows for more complex nonlinear dependencies of $\boldsymbol{y}$ within the objective function.

For a fixed $\boldsymbol{x} \in \mathbb{R}^p $, the \textit{inner minimization problem} is defined as:
\begin{equation}
    \min_{\boldsymbol{y} \in \mathbb{R}^q} F(\boldsymbol{x}, \boldsymbol{y}).
\end{equation}
Let $\boldsymbol{y}^{*}(\boldsymbol{x})$ denote a solution to this inner problem, which maps each $\boldsymbol{x}$ to a corresponding optimal choice of $\boldsymbol{y}$. The "elimination of linear variables" in VarPro is conceptually equivalent to solving this inner problem for $\boldsymbol{y}$ given $\boldsymbol{x}$.

With $\boldsymbol{y}^{*}(\boldsymbol{x})$ defined and assumed to be sufficiently smooth, the \textit{reduced optimization problem} is formulated by substituting $\boldsymbol{y}^{*}(\boldsymbol{x})$ back into the original objective function:
\begin{equation}
    \min_{\boldsymbol{x} \in \mathbb{R}^p} \tilde F(\boldsymbol{x}) := F(\boldsymbol{x}, \boldsymbol{y}^{*}(\boldsymbol{x})).
\end{equation}

\vspace{0.5cm}
\textbf{Our first question is:} \textit{What is the relationship between the stationary points of the original problem and those of the reduced problem when the variable elimination method is applied?}
\vspace{0.5cm}

The existence and uniqueness of $\boldsymbol{y}^{*}(\boldsymbol{x})$ are paramount for the reduced function $\tilde F(\boldsymbol{x})$ to be well-defined. If $\boldsymbol{y}^{*}(\boldsymbol{x})$ is not unique for some $\boldsymbol{x}$, the definition of $\tilde F(\boldsymbol{x})$ becomes ambiguous, potentially leading to a set-valued function or requiring a specific selection rule for $\boldsymbol{y}^{*}(\boldsymbol{x})$, which might complicate differentiability. Key properties of $\boldsymbol{y}^{*}(\boldsymbol{x})$ are outlined below.
 
 \textbf{Existence and Uniqueness:} For each $\boldsymbol{x}$, the existence of a solution $\boldsymbol{y}^{*}(\boldsymbol{x})$ to the inner problem can be ensured if $F(\boldsymbol{x},\cdot)$ is coercive (i.e., $F(\boldsymbol{x},\boldsymbol{y}) \to \infty$ as $\|\boldsymbol{y} \| \to \infty)$ and lower semi-continuous. Uniqueness is typically guaranteed if $F(\boldsymbol{x},\cdot)$ is strictly convex in $\boldsymbol{y}$ for all $\boldsymbol{x}$. These conditions ensure that for every $\boldsymbol{x}$, there is a unique optimal $\boldsymbol{y}$ value that minimizes $F(\boldsymbol{x},\boldsymbol{y})$.
 
 \textbf{Differentiability via Implicit Function Theorem (IFT):} The differentiability of $\boldsymbol{y}^{*}(\boldsymbol{x})$ is a cornerstone for deriving the optimality conditions of the reduced problem. If $F(\boldsymbol{x},\boldsymbol{y})$ is sufficiently smooth (e.g., $C^2$), and the inner problem satisfies certain regularity conditions, $\boldsymbol{y}^{*}(\boldsymbol{x})$ can be shown to be a $C^1$ function. The function $\boldsymbol{y}^{*}(\boldsymbol{x})$ may be not given by an explicit formula, but is defined implicitly as the solution to the first-order optimality condition of the inner problem:
\begin{equation}\nonumber
 \nabla _{\boldsymbol{y}} F(\boldsymbol{x}, \boldsymbol{y}^{*}(\boldsymbol{x})) = 0.
 \end{equation}
The Implicit Function Theorem (IFT) provides the precise conditions under which this equation guarantees that $\boldsymbol{y}^{*}$ is a unique and smoothly differentiable function of $\boldsymbol{x}$. The single most important condition is that the Hessian of the inner problem, $\nabla_{\boldsymbol{y}\boldsymbol{y}}^2 F$, must be invertible (non-singular) at the solution point $(\boldsymbol{x}, \boldsymbol{y}^{*}(\boldsymbol{x}))$.

Geometrically, an invertible Hessian means that the inner problem has a non-zero curvature at its minimum. The minimum isn't "flat"; it sits at the bottom of a well-defined bowl.  This "strong" minimum is stable: if you slightly change $\boldsymbol{x}$, the bottom of the bowl will shift smoothly. This smooth shifting of the minimum is what makes the function $\boldsymbol{y}^{*}(\boldsymbol{x})$ differentiable.
  
  \begin{theorem}[Critical Points Preservation (Reduced $\Rightarrow$ Original)] 
  Let $F: \mathbb{R}^p \times \mathbb{R}^q \to \mathbb{R}$ be a $C^1$ function. Assume that for each $\boldsymbol{x} \in \mathbb{R}^p$, the inner problem $\min_{\boldsymbol{y} \in \mathbb{R}^q} F(\boldsymbol{x},\boldsymbol{y})$ has a unique solution $\boldsymbol{y}^{*}(\boldsymbol{x})$, and that $\boldsymbol{y}^{*}(\boldsymbol{x})$ is a $C^1$ function. If $\boldsymbol{x}^*$ is a critical point of the reduced problem $\tilde F(\boldsymbol{x})$, then $(\boldsymbol{x}^*, \boldsymbol{y}^{*}(\boldsymbol{x}^*))$ is a critical point of the original problem $F(\boldsymbol{x},\boldsymbol{y})$.
  \end{theorem}
  
  \begin{proof}
Since is $\boldsymbol{x}^*$ a critical point of $\tilde F(\boldsymbol{x})$, we have: 
 \begin{equation}\nonumber
    \nabla \tilde F(\boldsymbol{x}^*) = 0.
\end{equation}

  The reduced function is defined as $\tilde F(\boldsymbol{x}) = F(\boldsymbol{x}, \boldsymbol{y}^{*}(\boldsymbol{x}))$ . By the chain rule, the gradient of $\tilde F(\boldsymbol{x})$ is given by:
 \begin{equation}\nonumber
    \nabla \tilde F(\boldsymbol{x}) = \nabla_{\boldsymbol{x}} F(\boldsymbol{x}, \boldsymbol{y}^{*}(\boldsymbol{x})) + \Bigg(\frac{\mathrm{d} \boldsymbol{y}^{*}(\boldsymbol{x})}{\mathrm{d} \boldsymbol{x}}\Bigg)^T \nabla_{\boldsymbol{y}} F(\boldsymbol{x}, \boldsymbol{y}^{*}(\boldsymbol{x})).
\end{equation}
By definition, $\boldsymbol{y}^{*}(\boldsymbol{x})$ is a minimizer of the inner problem, so it must satisfy the first-order optimality condition $\nabla_{\boldsymbol{y}} F(\boldsymbol{x}, \boldsymbol{y}^{*}(\boldsymbol{x})) = 0$. Therefore, the second term vanishes, and the gradient simplifies to:
\begin{equation}\nonumber
    \nabla \tilde F(\boldsymbol{x}) = \nabla_{\boldsymbol{x}} F(\boldsymbol{x}, \boldsymbol{y}^{*}(\boldsymbol{x})).
\end{equation}

 Evaluating this expression at $\boldsymbol{x}^*$, we get:
 \begin{equation}
    \nabla_{\boldsymbol{x}} F(\boldsymbol{x}^*, \boldsymbol{y}^{*}(\boldsymbol{x}^*)) = \nabla \tilde F(\boldsymbol{x}^*) =\boldsymbol 0.
 \end{equation}

 Since $\boldsymbol{y}^{*}(\boldsymbol{x}^*)$ is the solution to the inner problem $\min_{\boldsymbol{y} \in \mathbb{R}^q} F(\boldsymbol{x}^*,\boldsymbol{y})$, it must satisfy:
 \begin{equation}
    \nabla_{\boldsymbol{y}} F(\boldsymbol{x}^*, \boldsymbol{y}^{*}(\boldsymbol{x}^*))  =\boldsymbol 0.
 \end{equation}
 It follows from (17) and (18) that  $(\boldsymbol{x}^*, \boldsymbol{y}^{*}(\boldsymbol{x}^*))$ is, by definition, a critical point of the original problem $F(\boldsymbol{x},\boldsymbol{y})$.
 \end{proof}
 
 \vspace{0.5cm}
 \begin{theorem}[Global Minimizers Preservation (Reduced $\Rightarrow$ Original)]
  Let $F: \mathbb{R}^p \times \mathbb{R}^q \to \mathbb{R}$ be a continuous function. Assume that for each $\boldsymbol{x} \in \mathbb{R}^p$, the inner problem $\min_{\boldsymbol{y} \in \mathbb{R}^q} F(\boldsymbol{x},\boldsymbol{y})$ has a solution $\boldsymbol{y}^{*}(\boldsymbol{x})$. If $\boldsymbol{x}^*$ is a global minimizer of the reduced problem $\tilde F(\boldsymbol{x})$, then $(\boldsymbol{x}^*, \boldsymbol{y}^{*}(\boldsymbol{x}^*))$ is a global minimizer of the original objective $F(\boldsymbol{x},\boldsymbol{y})$.
\end{theorem}
  \begin{proof} 
  Let $\boldsymbol{x}^*$ be a global minimizer of $\tilde F(\boldsymbol{x})$. By definition, this means $\tilde F(\boldsymbol{x}^*) \leq \tilde F(\boldsymbol{x})$ for all $\boldsymbol{x} \in \mathbb{R}^p$. By the definition of the reduced function, $\tilde F(\boldsymbol{x}^*) = F(\boldsymbol{x}^*, \boldsymbol{y}^{*}(\boldsymbol{x}^*))$. 
  
  For any arbitrary point $(\boldsymbol{x},\boldsymbol{y}) \in \mathbb{R}^p \times \in \mathbb{R}^q$ in the domain of $F$, we know by the definition of $\boldsymbol{y}^{*}(\boldsymbol{x})$ (as a solution to the inner minimization problem for a fixed $\boldsymbol{x}$) that:
  \begin{equation}
    F(\boldsymbol{x},\boldsymbol{y}^{*}(\boldsymbol{x})) \leq F(\boldsymbol{x},\boldsymbol{y}).
 \end{equation}
 
  Combining these inequalities, we have:
  \begin{equation}
     F(\boldsymbol{x}^*, \boldsymbol{y}^{*}(\boldsymbol{x}^*)) \leq \tilde F(\boldsymbol{x}) = F(\boldsymbol{x},\boldsymbol{y}^{*}(\boldsymbol{x})) \leq F(\boldsymbol{x},\boldsymbol{y}).
  \end{equation}
  This sequence of inequalities holds for all $(\boldsymbol{x},\boldsymbol{y}) \in \mathbb{R}^p \times \in \mathbb{R}^q$, which means that $(\boldsymbol{x}^*, \boldsymbol{y}^{*}(\boldsymbol{x}^*))$ is a global minimizer of the original objective $F(\boldsymbol{x},\boldsymbol{y})$.
  \end{proof}
 \begin{remark} This is a powerful result, as finding a global minimizer in the reduced space directly guarantees a global minimizer in the original space. Notably, this theorem holds under weaker conditions than for critical points; it does not require uniqueness or differentiability of $\boldsymbol{y}^{*}(\boldsymbol{x})$, only its existence.	
 \end{remark}
 In the following, we investigate the converse relationship, specifically under what conditions a critical point of the original problem corresponds to a critical point of the reduced problem.
 
 \begin{theorem}[Critical Points Preservation (Original $\Rightarrow$ Reduced)]
 Let $F: \mathbb{R}^p \times \mathbb{R}^q \to \mathbb{R}$ be a $C^1$ function. Assume that for each $\boldsymbol{x} \in \mathbb{R}^p$, the inner problem $\min_{\boldsymbol{y} \in \mathbb{R}^q} F(\boldsymbol{x},\boldsymbol{y})$ has a unique solution $\boldsymbol{y}^{*}(\boldsymbol{x})$, and that $\boldsymbol{y}^{*}(\boldsymbol{x})$ is a $C^1$ function. If $(\tilde{\boldsymbol{x}}, \tilde{\boldsymbol{y}})$ is a critical point of the original problem $F(\boldsymbol{x},\boldsymbol{y})$ that also satisfies $\tilde{\boldsymbol{y}} = \boldsymbol{y}^{*}(\tilde{\boldsymbol{x}})$, then $\tilde{\boldsymbol{x}}$ is a critical point of the reduced problem $\tilde F(\boldsymbol{x})$.
 \end{theorem}
 
 \begin{proof} 
 By premise, $(\tilde{\boldsymbol{x}}, \tilde{\boldsymbol{y}})$ is a critical point of $F(\boldsymbol{x},\boldsymbol{y})$. This means both partial gradients of $F(\boldsymbol{x},\boldsymbol{y})$ are zero at this point:
\begin{align}
\nabla_{\boldsymbol{x}} F(\tilde{\boldsymbol{x}}, \tilde{\boldsymbol{y}}) = \boldsymbol{0},~~
\nabla_{\boldsymbol{y}} F(\tilde{\boldsymbol{x}}, \tilde{\boldsymbol{y}}) = \boldsymbol{0} \label{eq:grad_y_zero}.
\end{align}
 Using the chain rule, the gradient of $\tilde F(\boldsymbol{x})$ at $\tilde{\boldsymbol{x}}$ is:
  \begin{equation}
  \nabla \tilde F(\tilde{\boldsymbol{x}}) = \nabla_{\boldsymbol{x}} F (\tilde{\boldsymbol{x}}, \boldsymbol{y}^*(\tilde{\boldsymbol{x}})) + \left(\frac{\partial \boldsymbol{y}^*}{\partial \boldsymbol{x}}(\tilde{\boldsymbol{x}})\right)^T  \nabla_{\boldsymbol{y}} F(\tilde{\boldsymbol{x}}, \boldsymbol{y}^*(\tilde{\boldsymbol{x}})).
  \end{equation}
  Substituting $\tilde{\boldsymbol{y}} = \boldsymbol{y}^{*}(\tilde{\boldsymbol{x}})$ into this expression and using (21), we obtain:
  \begin{equation}
  \nabla \tilde F(\tilde{\boldsymbol{x}}) = \nabla_{\boldsymbol{x}} F (\tilde{\boldsymbol{x}}, \tilde{\boldsymbol{y}}) + \left(\frac{\partial \boldsymbol{y}^*}{\partial \boldsymbol{x}}(\tilde{\boldsymbol{x}})\right)^T \nabla_{\boldsymbol{y}} F (\tilde{\boldsymbol{x}}, \tilde{\boldsymbol{y}}) = \boldsymbol 0.
  \end{equation}
  This confirms that $\tilde{\boldsymbol{x}}$ is a critical point of the reduced problem $\tilde F(\boldsymbol{x})$.
  \end{proof}
  
 \begin{remark} Combining Theorems 1 and 3 establishes a crucial correspondence: under the given smoothness and uniqueness conditions, there is a one-to-one mapping between the critical points of the reduced problem $\tilde F(\boldsymbol{x})$ and a specific subset of critical points of the original problem $F(\boldsymbol{x},\boldsymbol{y})$.  This subset consists of all, and only, the critical points that lie on the optimal $p$-dimensional submanifold $\mathcal M=\{ (\boldsymbol{x}, \boldsymbol{y}^{*}(\boldsymbol{x}))| \boldsymbol{x} \in \mathbb{R}^p\}$. In this sense, the reduced problem is guaranteed to find every critical point of the original problem that satisfies the inner-problem's optimality condition.	
 \end{remark}

 It is important to clarify the scope of Theorem 3. The original problem may possess critical points $(\boldsymbol{x}_c,\boldsymbol{y}_c)$ for which $\boldsymbol{y}_c \neq \boldsymbol{y}^{*}(\boldsymbol{x}_c)$. These critical points do not lie on the optimal submanifold $\mathcal M$ and, therefore, have no corresponding critical point in the reduced landscape. They are effectively "filtered out" by the variable elimination process. We provide examples of such cases in Appendix B.

 \subsection{Reshaping Curvature: How Minima and Saddles are Transformed}
 \textbf{Our second question is:} \textit{Do the types of stationary points change when the original problem is transformed into the reduced problem?}
 \vspace{0.5cm}
 
 We begin with a proposition stating that under common conditions for some types of separable problems, the original objective function $F(\boldsymbol{x},\boldsymbol{y})$ does not possess any local maxima. We then analyze how the remaining stationary points---local minima and saddle points---are mapped between the original and reduced problems.

\begin{proposition}
Let $F: \mathbb{R}^p \times \mathbb{R}^q \to \mathbb{R}$ be a $C^2$ function. Assume that for each fixed $\boldsymbol{x} \in \mathbb{R}^p$, the function $F(\boldsymbol{x},\cdot)$ is \textbf{strictly convex} with respect to $\boldsymbol{y}$. Then the function $F(\boldsymbol{x},\boldsymbol{y})$ has no local maxima.
\end{proposition}

\begin{proof}
The proof proceeds by assuming that a local maximum exists and demonstrating that this leads to a contradiction with the uniqueness of the inner problem's solution.

    Assume, for the sake of contradiction, that there exists a point $(\boldsymbol{x}_0, \boldsymbol{y}_0)$ that is a local maximum of $F(\boldsymbol{x},\boldsymbol{y})$. As a local extremum, this point must be a critical point, so $\nabla F(\boldsymbol{x}_0, \boldsymbol{y}_0) = \boldsymbol{0}$.

    For $(\boldsymbol{x}_0, \boldsymbol{y}_0)$ to be a local maximum, its Hessian matrix, $\boldsymbol{H}_F(\boldsymbol{x}_0,\boldsymbol{y}_0)$, must be negative semi-definite. Let the block structure of the Hessian be:
    \[
        \boldsymbol{H}_F = \begin{pmatrix} \boldsymbol{A} & \boldsymbol{B}^T \\ \boldsymbol{B} & \boldsymbol{D} \end{pmatrix}, \quad \text{where } \boldsymbol{D} = \nabla_{\boldsymbol{y}\boldsymbol{y}}^2 F(\boldsymbol{x}_0,\boldsymbol{y}_0) \text{ and } \boldsymbol{B} = \nabla_{\boldsymbol{yx}}^2 F(\boldsymbol{x}_0,\boldsymbol{y}_0).
    \]
    The convexity assumption implies that $\boldsymbol{D}$ is positive semi-definite. However, for $\boldsymbol{H}_F$ to be negative semi-definite, its principal submatrix $\boldsymbol{D}$ must also be negative semi-definite. The only matrix that is simultaneously positive and negative semi-definite is the zero matrix. Thus, we must conclude:
    \begin{equation} \label{eq:D_zero}
        \boldsymbol{D} = \nabla_{\boldsymbol{y}\boldsymbol{y}}^2 F(\boldsymbol{x}_0,\boldsymbol{y}_0) = \boldsymbol{0}.
    \end{equation}
    Furthermore, for the full Hessian $\boldsymbol{H}_F$ to be negative semi-definite with $\boldsymbol{D}=\boldsymbol{0}$, the off-diagonal block $\boldsymbol{B}$ must also be the zero matrix. If not, for a given $\Delta\boldsymbol{x}$ where $\boldsymbol{B}\Delta\boldsymbol{x} \ne \boldsymbol{0}$, the term $2(\Delta\boldsymbol{x})^T\boldsymbol{B}^T(\Delta\boldsymbol{y})$ could be made arbitrarily large and positive, violating the negative semi-definite condition. Thus:
     \begin{equation} \label{eq:B_zero}
        \boldsymbol{B} = \nabla_{\boldsymbol{yx}}^2 F(\boldsymbol{x}_0,\boldsymbol{y}_0) = \boldsymbol{0}.
    \end{equation}

    The relationship between $\boldsymbol{x}$ and the unique solution $\boldsymbol{y}^*(\boldsymbol{x})$ is defined by the first-order condition $\boldsymbol{h}(\boldsymbol{x}, \boldsymbol{y}^*(\boldsymbol{x})) = \boldsymbol{0}$, where $\boldsymbol{h}(\boldsymbol{x}, \boldsymbol{y}) := \nabla_{\boldsymbol{y}} F(\boldsymbol{x}, \boldsymbol{y})$.

    The Implicit Function Theorem (IFT) provides conditions under which the solution function $\boldsymbol{y}^*(\boldsymbol{x})$ is guaranteed to exist locally and be well-behaved (e.g., continuous or differentiable). A standard condition for the IFT to apply at a point $(\boldsymbol{x}_0, \boldsymbol{y}_0)$ is that the Jacobian of $\boldsymbol{h}$ with respect to the variable being solved for ($\boldsymbol{y}$) must be invertible (non-singular). This Jacobian is precisely:
    \[
        \frac{\partial \boldsymbol{h}}{\partial \boldsymbol{y}}\bigg|_{(\boldsymbol{x}_0, \boldsymbol{y}_0)} = \nabla_{\boldsymbol{y}\boldsymbol{y}}^2 F(\boldsymbol{x}_0,\boldsymbol{y}_0).
    \]
    
    We have deduced from the existence of a local maximum that $\nabla_{\boldsymbol{y}\boldsymbol{y}}^2 F(\boldsymbol{x}_0,\boldsymbol{y}_0) = \boldsymbol{0}$, as shown in \eqref{eq:D_zero}. The zero matrix is singular and not invertible.

    This means the necessary condition for the Implicit Function Theorem is violated at $(\boldsymbol{x}_0, \boldsymbol{y}_0)$. The failure of this condition implies that the uniqueness of the inner solution is not robust to perturbations of $\boldsymbol{x}$ around $\boldsymbol{x}_0$. A small change in $\boldsymbol{x}$ could easily cause the solution $\boldsymbol{y}^*(\boldsymbol{x})$ to bifurcate or become non-unique. This local instability at $\boldsymbol{x}_0$ is fundamentally incompatible with the global assumption that a unique, stable solution path $\boldsymbol{y}^*(\boldsymbol{x})$ exists for every $\boldsymbol{x}$ in the domain. Therefore, the initial assumption must be false, and the function $F(\boldsymbol{x},\boldsymbol{y})$ can have no local maxima.
\end{proof}

To analyze the Hessian of the reduced problem, we first introduce the Schur complement \citep{ouellette:81}, which will be a key analytical tool.

\begin{proposition}[Schur Complement Condition for Definiteness \citep{wiki:schur}]
Let $\boldsymbol{H}$ be a symmetric block matrix defined as:
\[
\boldsymbol{H} = \begin{pmatrix} \boldsymbol{A} & \boldsymbol{B}^T \\ \boldsymbol{B} & \boldsymbol{D} \end{pmatrix}
\]
where $\boldsymbol{A} = \boldsymbol{A}^T$ and $\boldsymbol{D} = \boldsymbol{D}^T$. If $\boldsymbol{D}$ is positive definite ($\boldsymbol{D} \succ 0$), then $\boldsymbol{H}$ is positive definite ($\boldsymbol{H} \succ 0$) if and only if its Schur complement, $\boldsymbol{S} = \boldsymbol{A} - \boldsymbol{B}^T\boldsymbol{D}^{-1}\boldsymbol{B}$, is positive definite ($\boldsymbol{S} \succ 0$).
\end{proposition}

The Hessian of the reduced problem, $\nabla^2 \tilde{F}(\boldsymbol{x})$, is precisely the Schur complement of the Hessian of the original problem, $\boldsymbol{H}_F$, evaluated at the corresponding stationary point $(\boldsymbol{x}, \boldsymbol{y}^*(\boldsymbol{x}))$. Specifically, if we partition $\boldsymbol{H}_F$ as:
\[
\boldsymbol{H}_F = \begin{pmatrix} \nabla_{\boldsymbol{x}\boldsymbol{x}}^2 F & \nabla_{\boldsymbol{y}\boldsymbol{x}}^2 F^T \\ \nabla_{\boldsymbol{y}\boldsymbol{x}}^2 F & \nabla_{\boldsymbol{y}\boldsymbol{y}}^2 F \end{pmatrix} = \begin{pmatrix} \boldsymbol{A} & \boldsymbol{B}^T \\ \boldsymbol{B} & \boldsymbol{D}\end{pmatrix},
\]
then it can be shown that $\nabla^2 \tilde{F}(\boldsymbol{x}) = \boldsymbol{A} - \boldsymbol{B}^T\boldsymbol{D}^{-1}\boldsymbol{B}$. This connection is fundamental to the following theorems.

\begin{theorem}[Preservation of Minima]
Let $F: \mathbb{R}^p \times \mathbb{R}^q \to \mathbb{R}$ be a $C^2$ function, and assume that for each $\boldsymbol{x} \in \mathbb{R}^p$, the Hessian of the inner problem, $\nabla_{\boldsymbol{y}\boldsymbol{y}}^2 F(\boldsymbol{x},\cdot)$, is positive definite.\\
(i) If $(\boldsymbol{x}_0, \boldsymbol{y}_0)$ is a local minimum of the original problem $F(\boldsymbol{x}, \boldsymbol{y})$ and its hessian is positive definite, then $\boldsymbol{x}_0$ is a local minimum of the reduced problem $\tilde{F}(\boldsymbol{x})$.\\
(ii) If $\boldsymbol{x}_0$ is a local minimum of the reduced problem $\tilde{F}(\boldsymbol{x})$ and its hessian is positive definite, then $(\boldsymbol{x}_0, \boldsymbol{y}_0)$ is a local minimum of the original problem $F(\boldsymbol{x}, \boldsymbol{y})$.

\end{theorem}

\begin{proof}
The proof relies on the fundamental relationship between the Hessian of the original problem, $\boldsymbol{H}_F$, and the Hessian of the reduced problem, $\nabla^2\tilde{F}$, which is the Schur complement of $\boldsymbol{H}_F$.

Let $(\boldsymbol{x}_0, \boldsymbol{y}_0)$ be a stationary point of $F$. By Theorems 1 and 3, this is equivalent to $\boldsymbol{x}_0$ being a stationary point of $\tilde{F}$. To determine the character of these points, we examine their Hessians. Let the Hessian of the original problem at this point be:
\[
\boldsymbol{H}_F = \begin{pmatrix} \boldsymbol{A} & \boldsymbol{B}^T \\ \boldsymbol{B} & \boldsymbol{D} \end{pmatrix},
\]
where $\boldsymbol{A} = \nabla_{\boldsymbol{xx}}^2 F$, $\boldsymbol{B} = \nabla_{\boldsymbol{yx}}^2 F$, and $\boldsymbol{D} = \nabla_{\boldsymbol{yy}}^2 F$. The Hessian of the reduced problem is its Schur complement, $\nabla^2 \tilde{F}(\boldsymbol{x}_0) = \boldsymbol{A} - \boldsymbol{B}^T\boldsymbol{D}^{-1}\boldsymbol{B}$.

By assumption, the block $\boldsymbol{D}$ is positive definite ($\boldsymbol{D} \succ 0$). According to the Schur complement condition for definiteness (Proposition 2), the full matrix $\boldsymbol{H}_F$ is positive definite if and only if its Schur complement is positive definite. That is:
\[
\boldsymbol{H}_F \succ 0 \iff \nabla^2 \tilde{F}(\boldsymbol{x}_0) \succ 0.
\]
This direct equivalence proves both directions of the theorem:

($\Rightarrow$) Assume $(\boldsymbol{x}_0, \boldsymbol{y}_0)$ is a local minimum of $F$ and its Hessian $\boldsymbol{H}_F$ is positive definite. By the equivalence, the Hessian of the reduced problem $\nabla^2 \tilde{F}(\boldsymbol{x}_0)$ must also be positive definite. Since $\boldsymbol{x}_0$ is a stationary point with a positive definite Hessian, it is a local minimum of $\tilde{F}(\boldsymbol{x})$.

($\Leftarrow$) Assume $\boldsymbol{x}_0$ is a local minimum of $\tilde{F}$ and its Hessian $\nabla^2 \tilde{F}(\boldsymbol{x}_0)$ is positive definite. By the equivalence, the Hessian of the original problem $\boldsymbol{H}_F$ must also be positive definite. Since $(\boldsymbol{x}_0, \boldsymbol{y}_0)$ is a stationary point with a positive definite Hessian, it is a local minimum of $F(\boldsymbol{x},\boldsymbol{y})$.
\end{proof}

\begin{remark}
A powerful geometric intuition underpins Theorem 4. The variable elimination method can be viewed as restricting the optimization of $F(\boldsymbol{x}, \boldsymbol{y})$ to the $p$-dimensional submanifold $\mathcal{M} = \left \{(\boldsymbol{x}, \boldsymbol{y})| \boldsymbol{y} = \boldsymbol{y}^*(\boldsymbol{x}) \right \}$. The reduced function, $\tilde{F}(\boldsymbol{x}) = F(\boldsymbol{x}, \boldsymbol{y}^*(\boldsymbol{x}))$, is precisely the original function $F$ evaluated on this manifold.

This perspective provides a direct proof for Theorem $4(i)$: if a point is a local minimum in the full space $\mathbb{R}^{p+q}$, it must also be a local minimum when restricted to any submanifold passing through it. This definitional argument is more general than the Hessian-based proof, as it does not require the Hessian to be positive definite. See Appendix C for further details.

Crucially, this one-way preservation of minima implies that the transformation cannot create "new" minima. A stationary point $(\boldsymbol{x}_0, \boldsymbol{y}^*(\boldsymbol{x}_0))$ that is a saddle point of the original problem is, by definition, not a local minimum. Therefore, its corresponding point cannot be a local minimum of the reduced problem. This reinforces our central finding: variable elimination simplifies the optimization landscape by altering or removing undesirable stationary points, not by introducing spurious local minima. 
\end{remark}

\begin{theorem}[Transformation of Saddle Points to Local Maxima]
Let $F: \mathbb{R}^p \times \mathbb{R}^q \to \mathbb{R}$ be a $C^2$ function. Assume that for each $\boldsymbol{x} \in \mathbb{R}^p$, the inner problem $\min_{\boldsymbol{y} \in \mathbb{R}^q} F(\boldsymbol{x},\boldsymbol{y})$ is \textbf{strictly convex} with positive definite hessian.\\
(i) If $\boldsymbol{x}_0$ is a local maximum of the reduced objective $\tilde{F}(\boldsymbol{x})$, then its corresponding point $(\boldsymbol{x}_0, \boldsymbol{y}^*(\boldsymbol{x}_0))$ is a \textbf{saddle point} of the original objective $F(\boldsymbol{x},\boldsymbol{y})$.\\
(ii) Conversely, let $(\boldsymbol{x}_0, \boldsymbol{y}_0)$ be a stationary point of $F$ with $\boldsymbol{y}_0 = \boldsymbol{y}^*(\boldsymbol{x}_0)$. If the function of $\boldsymbol{x}$ obtained by fixing $\boldsymbol{y}=\boldsymbol{y}_0$, denoted $g(\boldsymbol{x}) = F(\boldsymbol{x}, \boldsymbol{y}_0)$, has a local maximum at $\boldsymbol{x}_0$ with $\nabla_{\boldsymbol{xx}}^2 F(\boldsymbol{x}_0, \boldsymbol{y}_0) \prec \boldsymbol{0}$, then:
\begin{itemize}
   \item[(a)] The point $\boldsymbol{x}_0$ is a strict \textbf{local maximum} of the reduced objective $\tilde{F}(\boldsymbol{x})$.
    \item[(b)] The point $(\boldsymbol{x}_0, \boldsymbol{y}_0)$ is a \textbf{saddle point} of the full objective $F(\boldsymbol{x},\boldsymbol{y})$.
\end{itemize}
\end{theorem}

\begin{proof}
\textbf{Proof of (i):}
This proof demonstrates that a local maximum in the reduced space must correspond to a saddle point in the original space. 

Let $\boldsymbol{x}_0$ be a local maximum of $\tilde{F}(\boldsymbol{x})$. By Theorem 1, the point $( (\boldsymbol{x}_0, \boldsymbol{y}^*(\boldsymbol{x}_0))$ is a stationary point of the original problem $F(\boldsymbol{x},\boldsymbol{y})$. From the strict convexity of the inner problem, we know from Proposition 1 that $F$ has no local maxima. Thus, $(\boldsymbol{x}_0, \boldsymbol{y}_0)$ cannot be a local maximum. 

We now show it cannot be a local minimum by contradiction. Assume $(\boldsymbol{x}_0, \boldsymbol{y}_0)$ is a local minimum of $F$. By Theorem 4, this would imply that $\boldsymbol{x}_0$ is a local minimum of $\tilde{F}$. This contradicts our premise that $\boldsymbol{x}_0$ is a local maximum.

Since the critical point $(\boldsymbol{x}_0, \boldsymbol{y}_0)$ is neither a local maximum nor a local minimum, it must be a saddle point.

\vspace{0.4cm}
\textbf{Proof of (ii):}
This proof shows that a saddle point that is maximal in its $\boldsymbol{x}$-directions transforms into a local maximum in the reduced space. 

(a) We determine the character of the stationary point $\boldsymbol{x}_0$ of $\tilde{F}$ by examining its Hessian, $\nabla^2 \tilde{F}(\boldsymbol{x}_0)$, which is the Schur complement of $\boldsymbol{H}_F$:
\[ \nabla^2 \tilde{F}(\boldsymbol{x}_0) = \nabla_{\boldsymbol{xx}}^2 F - (\nabla_{\boldsymbol{yx}}^2 F)^T (\nabla_{\boldsymbol{y}\boldsymbol{y}}^2 F)^{-1} (\nabla_{\boldsymbol{yx}}^2 F). \]
Let's analyze the terms on the right-hand side, evaluated at $(\boldsymbol{x}_0, \boldsymbol{y}_0)$:
\begin{itemize}
    \item The first term, $\nabla_{\boldsymbol{xx}}^2 F$, is \textbf{negative definite} by our premise.
    \item The second term, $(\nabla_{\boldsymbol{yx}}^2 F)^T (\nabla_{\boldsymbol{y}\boldsymbol{y}}^2 F)^{-1} (\nabla_{\boldsymbol{yx}}^2 F)$, is a \textbf{positive semi-definite} matrix.
\end{itemize}
The entire expression is a negative definite matrix minus a positive semi-definite matrix. This operation always results in a negative definite matrix.

Therefore, the Hessian of the reduced problem, $\nabla^2 \tilde{F}(\boldsymbol{x}_0)$, is negative definite. Since $\boldsymbol{x}_0$ is a stationary point of $\tilde{F}$ with a negative definite Hessian, it is a strict local maximum. This proves conclusion (a).

(b) Since $\boldsymbol{x}_0$ is a strict local maximum of the reduced objective $\tilde{F}(\boldsymbol{x})$, by Theorem 5(i) $(\boldsymbol{x}_0, \boldsymbol{y}_0)$ is a saddle point of the full objective $F(\boldsymbol{x},\boldsymbol{y})$.
\vspace{0.2cm}
\end{proof}

\begin{remark}
Theorem 5 provides a profound insight into the power of the variable elimination approach. Modern optimization theory suggests that the primary challenge in many large-scale non-convex problems is not the presence of poor-quality local minima, but the proliferation of saddle points that can dramatically slow down or stall first-order optimization algorithms. This theorem reveals how variable elimination reshapes these critical points. It establishes that \textbf{the local maxima encountered in the reduced landscape are, in fact, the benign representation of a specific class of saddle points from the original, high-dimensional space.} Local maxima are trivial to escape for many algorithms; for instance, a gradient-based method will naturally move away from them. By transforming a class of challenging saddle points into simple, repelling local maxima, variable elimination fundamentally simplifies the optimization landscape, robustly guiding algorithms toward higher-quality solutions.
\end{remark}

\subsection{A Unifying Viewpoint: Hessian Inertia}

While the preceding theorems establish the correspondence between specific types of stationary points, a more general and powerful understanding can be achieved by analyzing the \emph{inertia} of the Hessian matrices. This provides a complete description of how the local curvature of the problem is fundamentally altered by variable elimination.

For a real symmetric matrix $\boldsymbol{A}$, its inertia is an ordered triplet, $\text{In}(\boldsymbol{A}) = (n_+, n_-, n_0)$, representing the number of positive, negative, and zero eigenvalues, respectively. The inertia precisely defines the character of a non-degenerate stationary point: a local minimum corresponds to an inertia of $(n, 0, 0)$, while a saddle point has $n_- > 0$.

A fundamental result connecting the Hessians is the \textbf{Haynsworth inertia additivity formula} \citep{haynsworth:68}. It applies to the Schur complement, and since the Hessian of the reduced problem, $\nabla^2 \tilde{F}$, \emph{is} the Schur complement of the inner problem's Hessian, $\boldsymbol{D} = \nabla^2_{\boldsymbol{y}\boldsymbol{y}}F$, within the full Hessian $\boldsymbol{H}_F$, we have:
\begin{equation}
    \text{In}(\boldsymbol{H}_F) = \text{In}(\boldsymbol{D}) + \text{In}(\nabla^2 \tilde{F}(\boldsymbol{x}))
\end{equation}

In our setting, the assumption of strong convexity for the inner problem means $\boldsymbol{D}$ is always positive definite. Its inertia is therefore fixed as $\text{In}(\boldsymbol{D}) = (q, 0, 0)$, where $q$ is the dimension of $\boldsymbol{y}$. This yields a remarkably simple yet powerful relationship:
\begin{equation}
    \label{eq:inertia_simplified}
    \text{In}(\boldsymbol{H}_F) = (q, 0, 0) + \text{In}(\nabla^2 \tilde{F}(\boldsymbol{x}))
\end{equation}

This equation provides a unifying explanation for all our previous findings. Let $\text{In}(\boldsymbol{H}_F) = (n_+, n_-, n_0)$ and $\text{In}(\nabla^2 \tilde{F}(\boldsymbol{x})) = (p_+, p_-, p_0)$. Equation \eqref{eq:inertia_simplified} implies $(n_+, n_-, n_0) = (q+p_+, p_-, p_0)$. This leads to two profound insights:

\begin{enumerate}
    \item \textbf{Conservation of Negative Curvature:} Critically, the number of negative eigenvalues is perfectly conserved between the original and reduced problems:
    \[
    n_-(\boldsymbol{H}_F) = p_-(\nabla^2 \tilde{F}(\boldsymbol{x}))
    \]
    This immediately explains why the transformation is stable. A local minimum ($p_-=0$) in the reduced space can only correspond to a point with $n_-=0$ (a minimum) in the original space; it can never arise from a saddle point ($n_->0$). Similarly, variable elimination does not create "new" unstable directions related to the primary variables; it only preserves them.

    \item \textbf{Systematic Landscape Simplification:} Equation \eqref{eq:inertia_simplified} reveals that variable elimination simplifies the optimization landscape by methodically "factoring out" the $q$ dimensions of guaranteed positive curvature from the inner problem. This systematically "purifies" the character of the stationary points. We can see how this single rule explains all our previous theorems:
    \begin{itemize}
        \item \textbf{If $\boldsymbol{x}_0$ is a local minimum of $\tilde{F}$}: $\text{In}(\nabla^2 \tilde{F})=(p,0,0)$. Then $\text{In}(\boldsymbol{H}_F) = (q,0,0) + (p,0,0) = (p+q, 0, 0)$. The corresponding point is a local minimum.
        
        \item \textbf{If $\boldsymbol{x}_0$ is a local maximum of $\tilde{F}$}: $\text{In}(\nabla^2 \tilde{F})=(0,p,0)$. Then $\text{In}(\boldsymbol{H}_F) = (q,0,0) + (0,p,0) = (q, p, 0)$. The point has both positive and negative curvature, hence it is a saddle point.
        
        \item \textbf{If $\boldsymbol{x}_0$ is a saddle point of $\tilde{F}$}: $\text{In}(\nabla^2 \tilde{F})=(p_+,p_-,p_0)$ with $p_->0$. Then $\text{In}(\boldsymbol{H}_F) = (q+p_+, p_-, p_0)$. The point also has $n_->0$, hence it remains a saddle point.
    \end{itemize}
\end{enumerate}

\section{Examples: Non-convex Matrix Factorization}
We illustrate our theoretical results, particularly Theorem~5 (Transformation of Saddle Points to Local Maxima), by exploring a simple yet non-trivial problem -- non-convex matrix factorization \citep{zhu:18, li:19, chi:19, valavi:20}. This problem is commonly used to study the algorithmic behavior of gradient descent methods \citep{zhu:19}, including phenomena such as implicit regularization \citep{du:18, liu:21} and the escape from saddle points \citep{jin:17}. 

Given a matrix $\boldsymbol{M}\in \mathbb{R}^{d_1 \times d_2}$, the general form of the problem is:
\begin{equation}
    \min_{\boldsymbol{X}\in \mathbb{R}^{d_1 \times r}, \boldsymbol{Y}\in \mathbb{R}^{d_2 \times r}} \frac{1}{2}\| \boldsymbol{X} \boldsymbol{Y}^T - \boldsymbol{M} \|_F^2.
\end{equation}
This objective has a bilinear structure and is separable, as the minimization is a convex least-squares problem in $\boldsymbol{Y}$ for a fixed $\boldsymbol{X}$, and vice versa. Despite its simple formulation, the joint landscape is non-convex and exhibits a complex geometry, including the presence of non-trivial saddle points.

\subsection{Rank-1 Case}
We consider the following non-convex optimization:
 \begin{equation}
       \min_{\boldsymbol{x}\in \Re^{d_1}, \boldsymbol{y}\in \Re^{d_2}} F(\boldsymbol{x}, \boldsymbol{y}) = \frac{1}{2}\| \boldsymbol x \boldsymbol y^T  - \boldsymbol M \|_F^2,
  \end{equation}
where $\boldsymbol{M}\in \Re^{d_1 \times d_2}$ is a rank-1 matrix and can be factorized as follows: $\boldsymbol{M} = \boldsymbol{u}_*\boldsymbol{v}_*^T$, where $\boldsymbol{u}_* \in \Re^{d_1}$, $\boldsymbol{v}_* \in \Re^{d_2}$. Without loss of generality, we assume $\| \boldsymbol{u}_* \| = \| \boldsymbol{v}_* \| = 1$.
$F$ has two types of stationary points: \\
(i) global optimum: $(\alpha \boldsymbol{u}_*, \alpha^{-1}\boldsymbol{v}_*,)$; \\
(ii) saddle points: any $(\boldsymbol{x},0)$ and $(0, \boldsymbol{y})$ that satisfy $\boldsymbol{x}^T\boldsymbol{u}_* = \boldsymbol{y}^T\boldsymbol{v}_* = \boldsymbol 0$.

For any fixed $\boldsymbol{x} \neq 0$, the optimal solution for the inner problem $\arg \min_{\boldsymbol{y}} F(\boldsymbol{x}, \boldsymbol{y})$ is 
 \begin{equation}
   \boldsymbol{y}^*(\boldsymbol{x}) = \frac{\boldsymbol{M}^T\boldsymbol{x}}{\|\boldsymbol{x}\|^2 }.
 \end{equation}
Substituting (30) into the original objective function (29), we obtain the following reduced objective function:
\begin{equation}
   f(\boldsymbol{x}) = \frac{1}{2}\left  \| \frac{\boldsymbol x \boldsymbol{x}^T \boldsymbol{M}}{\|\boldsymbol{x}\|^2}  - \boldsymbol M \right\|_F^2.
\end{equation}

 This can be simplified by letting $\boldsymbol{P}_{\boldsymbol{x}} = \frac{\boldsymbol{x}\boldsymbol{x}^T}{\|\boldsymbol{x}\|_2^2}$ be the projection matrix onto the span of $\boldsymbol{x}$:
\begin{align}
    f(\boldsymbol{x}) = \frac{1}{2}\| \boldsymbol{P}_{\boldsymbol{x}}\boldsymbol{M} - \boldsymbol{M} \|_F^2 = \frac{1}{2}\| (\boldsymbol{P}_{\boldsymbol{x}} - \boldsymbol{I})\boldsymbol{M} \|_F^2. \label{eq:reduced_mf}
\end{align}
The reduced problem seeks to find a direction $\boldsymbol{x}$ that minimizes the approximation error after projecting $\boldsymbol{M}$ onto that direction.

We now show that the saddle points of the original problem are transformed into local maxima in the reduced problem.

\begin{proof}
Let's consider a saddle point of the original problem of the form $(\boldsymbol{x}_s, \boldsymbol{0})$, where $\boldsymbol{x}_s \neq \boldsymbol{0}$ and $\boldsymbol{x}_s^T\boldsymbol{u}_* = 0$. The corresponding point in the reduced landscape is $\boldsymbol{x}_s$. Our goal is to show that $\boldsymbol{x}_s$ is a local maximum of the reduced function $f(\boldsymbol{x})$.

First, we can further simplify the expression for $f(\boldsymbol{x})$ using $\boldsymbol{M} = \boldsymbol{u}_*\boldsymbol{v}_*^T$ and the fact that $\|\boldsymbol{u}_*\|_2 = \|\boldsymbol{v}_*\|_2=1$:
\begin{align*}
    f(\boldsymbol{x}) &= \frac{1}{2}\| (\boldsymbol{I} - \boldsymbol{P}_{\boldsymbol{x}})\boldsymbol{M} \|_F^2 = \frac{1}{2} \text{Tr}\left( \boldsymbol{M}^T(\boldsymbol{I} - \boldsymbol{P}_{\boldsymbol{x}})^T(\boldsymbol{I} - \boldsymbol{P}_{\boldsymbol{x}})\boldsymbol{M} \right) \\
    &= \frac{1}{2} \text{Tr}\left( \boldsymbol{M}^T(\boldsymbol{I} - \boldsymbol{P}_{\boldsymbol{x}})\boldsymbol{M} \right) \quad (\text{since } \boldsymbol{I}-\boldsymbol{P_x} \text{ is a projection}) \\
    &= \frac{1}{2} \left( \text{Tr}(\boldsymbol{M}^T\boldsymbol{M}) - \text{Tr}(\boldsymbol{M}^T\boldsymbol{P}_{\boldsymbol{x}}\boldsymbol{M}) \right) \\
    &= \frac{1}{2} \left( \|\boldsymbol{M}\|_F^2 - \left\|\boldsymbol{P}_{\boldsymbol{x}}\boldsymbol{M}\right\|_F^2 \right) \\
    &= \frac{1}{2} \left( 1 - \frac{\|\boldsymbol{x}\boldsymbol{x}^T\boldsymbol{u}_*\boldsymbol{v}_*^T\|_F^2}{\|\boldsymbol{x}\|_2^4} \right) = \frac{1}{2} \left( 1 - \frac{|\boldsymbol{x}^T\boldsymbol{u}_*|^2 \|\boldsymbol{x}\boldsymbol{v}_*^T\|_F^2}{\|\boldsymbol{x}\|_2^4} \right) \\
    &= \frac{1}{2} \left( 1 - \frac{(\boldsymbol{x}^T\boldsymbol{u}_*)^2 \|\boldsymbol{x}\|_2^2 \|\boldsymbol{v}_*\|_2^2}{\|\boldsymbol{x}\|_2^4} \right) = \frac{1}{2} \left( 1 - \frac{(\boldsymbol{x}^T\boldsymbol{u}_*)^2}{\|\boldsymbol{x}\|_2^2} \right).
\end{align*}
The term $\frac{(\boldsymbol{x}^T\boldsymbol{u}_*)^2}{\|\boldsymbol{x}\|_2^2}$ is the squared cosine of the angle between $\boldsymbol{x}$ and $\boldsymbol{u}_*$. Let this be $\cos^2(\theta_{\boldsymbol{x},\boldsymbol{u}_*})$. So, $f(\boldsymbol{x}) = \frac{1}{2}\sin^2(\theta_{\boldsymbol{x},\boldsymbol{u}_*})$.

Now, we evaluate $f(\boldsymbol{x})$ at the point $\boldsymbol{x}_s$. By definition of this saddle point, $\boldsymbol{x}_s^T\boldsymbol{u}_* = 0$. Therefore,
\[
f(\boldsymbol{x}_s) = \frac{1}{2} \left( 1 - \frac{0^2}{\|\boldsymbol{x}_s\|_2^2} \right) = \frac{1}{2}.
\]
For any arbitrary point $\boldsymbol{x} \in \mathbb{R}^{d_1}$, the term $\frac{(\boldsymbol{x}^T\boldsymbol{u}_*)^2}{\|\boldsymbol{x}\|_2^2} \ge 0$. Thus,
\[
f(\boldsymbol{x}) = \frac{1}{2} \left( 1 - \frac{(\boldsymbol{x}^T\boldsymbol{u}_*)^2}{\|\boldsymbol{x}\|_2^2} \right) \le \frac{1}{2}.
\]
This shows that $f(\boldsymbol{x}_s) = 1/2$ is the global maximum value of the reduced function $f(\boldsymbol{x})$. Therefore, the point $\boldsymbol{x}_s$ is not just a local maximum, but a global maximum of the reduced problem. This explicitly illustrates how the variable elimination method transforms a saddle point of the original problem into a well-behaved maximizer in the reduced landscape, confirming the conclusion of Theorem~5.
\end{proof}

To provide a concrete and visual illustration of our theoretical findings, we analyze a simple 2$\times$2 matrix factorization problem.

Consider the rank-1 matrix factorization problem where the target matrix ($\boldsymbol{u}_* = (1, 0)^T$, $\boldsymbol{v}_* = (0, 1)^T$)is:
\[
\boldsymbol{M} = \begin{pmatrix} 0 & 1 \\ 0 & 0 \end{pmatrix}.
\]
The original optimization problem is to find $\boldsymbol{x} \in \mathbb{R}^2$ and $\boldsymbol{y} \in \mathbb{R}^2$ that solve:
\begin{equation}
    \min_{\boldsymbol{x}, \boldsymbol{y}} F(\boldsymbol{x}, \boldsymbol{y}) = \frac{1}{2}\left\| \boldsymbol{x}\boldsymbol{y}^T - \begin{pmatrix} 0 & 1 \\ 0 & 0 \end{pmatrix} \right\|_F^2.
\end{equation}

The stationary points ($\boldsymbol{x} \neq 0$) of this original problem are:
\begin{itemize}
    \item \textbf{Global Minima:} The set of points $(\alpha\boldsymbol{u}_*, \alpha^{-1}\boldsymbol{v}_*)$ for any $\alpha \neq 0$.
    \item \textbf{Saddle Points:} The set of points $(\boldsymbol{x}_s, \boldsymbol{0})$ where $\boldsymbol{x}_s$ is orthogonal to $\boldsymbol{u}_*$. In this case, any vector of the form $(0, c)^T$ for $c \neq 0$.
\end{itemize}

By applying variable elimination, the reduced objective function $f(\boldsymbol{x}) = \frac{1}{2}\| (\boldsymbol{I} - \boldsymbol{P_x})\boldsymbol{M} \|_F^2$ simplifies to:
\begin{equation}
    f(\boldsymbol{x}) = f(x_1, x_2) = \frac{1}{2} \left( 1 - \frac{x_1^2}{x_1^2 + x_2^2} \right).
\end{equation}
The landscape of this reduced function is plotted in Figure \ref{fig:landscape}.

\begin{figure}[h!]
    \centering
    \includegraphics[height=0.6\textwidth]{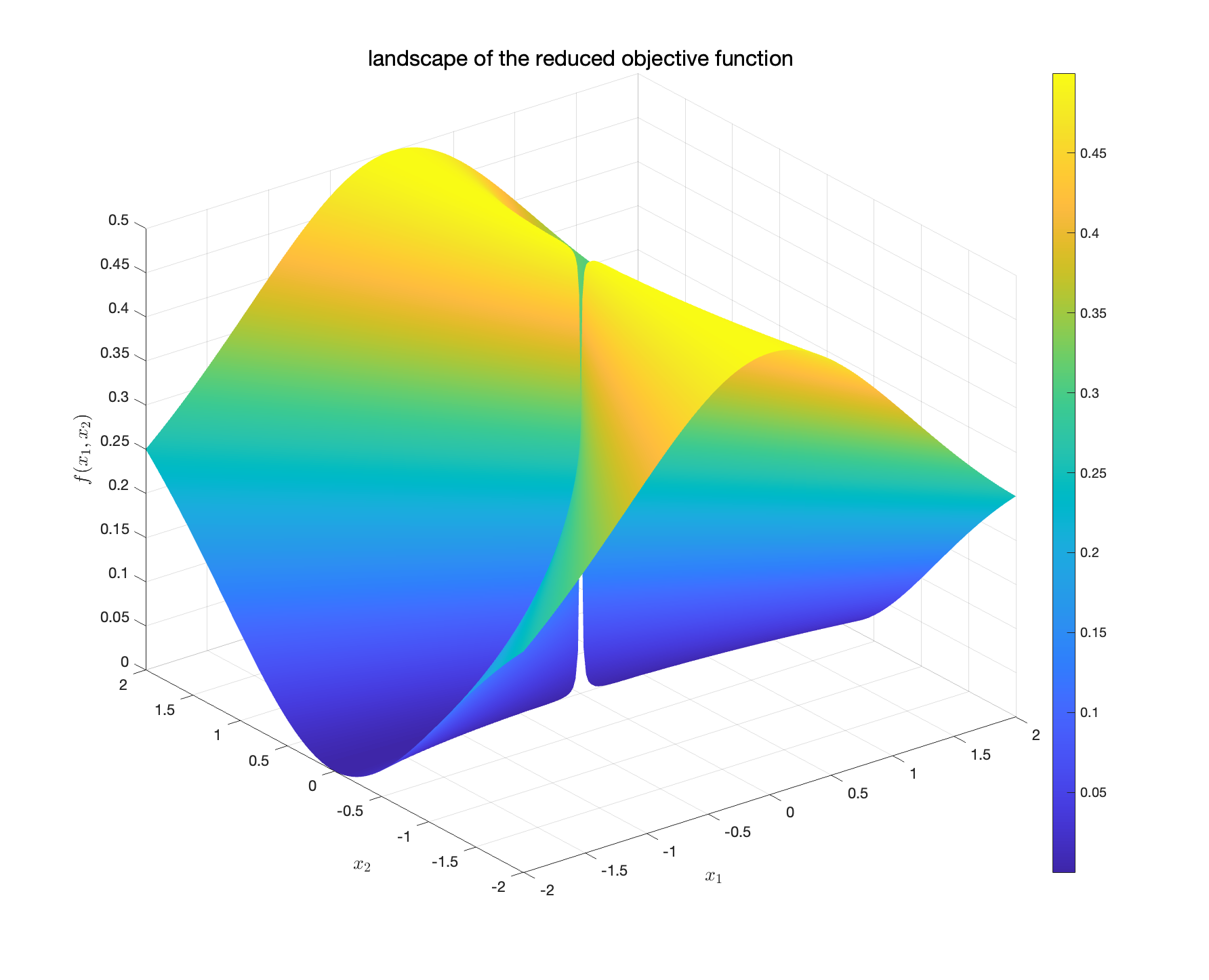}
    \caption{The 3D landscape of the reduced objective function $f(\boldsymbol{x})$ for the matrix $\boldsymbol{M} = (0,1; 0,0)$. The function's value depends only on the direction of $\boldsymbol{x}$, not its magnitude. The valley along the $x_1$-axis corresponds to the global minima, while the ridge along the $x_2$-axis corresponds to the global maxima.}
    \label{fig:landscape}
\end{figure}

This 3D plot provides a clear visualization of our theoretical findings and perfectly illustrates the conclusion of Theorem 5.

\begin{itemize}
    \item \textbf{Landscape Ridges and Valleys:}    \begin{itemize}
        \item \textbf{Valleys (Global Minima):} The lowest points on the surface (where $f(\boldsymbol{x})=0$) are along the $x_1$-axis (where $x_2=0$). This corresponds to the direction of the principal left singular vector $\boldsymbol{u}_*=(1,0)^T$. When $\boldsymbol{x}$ aligns with this "true" signal direction, the reduced problem finds the optimal solution, perfectly illustrating Theorem 4 (Preservation of Minima).
        \item \textbf{Ridges (Global Maxima):} The highest points on the surface (where $f(\boldsymbol{x})=0.5$) are along the $x_2$-axis (where $x_1=0$). This direction, which is spanned by the vector $(0,1)^T$, is precisely the subspace orthogonal to $\boldsymbol{u}_*$.
    \end{itemize}
    \item \textbf{Verification of Saddle Point Transformation:}
    In the original problem, the saddle points we analyzed are of the form $(\boldsymbol{x}_s, \boldsymbol{0})$, where $\boldsymbol{x}_s$ is orthogonal to $\boldsymbol{u}_*$. In this 2D example, this means $\boldsymbol{x}_s$ must lie on the $x_2$-axis. As clearly shown in the plot, the entire $x_2$-axis forms the ridge of global maxima for the reduced function $f(\boldsymbol{x})$.
\end{itemize}
This perfectly confirms our theory: the saddle points of the original problem have been transformed into the global maxima of the reduced problem. This provides a clear geometric reason why the variable elimination approach is effective at avoiding these specific saddles --- an optimization algorithm operating on the reduced landscape will be naturally repelled from these maximal ridges.

\subsection{The Rank-$r$ Case: A Grassmannian Perspective}
We now extend the analysis to the more general rank-$r$ case, adopting the language of \textbf{Grassmannian manifolds} to offer a more profound geometric interpretation. This perspective is motivated by the core insight from \cite{valavi:20}, which posits that the reduced landscape is intrinsically defined on the manifold of all $r$-dimensional subspaces.

The original problem is formulated as:
\begin{equation}
    \min_{\boldsymbol{X}\in \mathbb{R}^{d_1 \times r}, \boldsymbol{Y}\in \mathbb{R}^{d_2 \times r}} F(\boldsymbol{X}, \boldsymbol{Y}) = \frac{1}{2}\| \boldsymbol{X} \boldsymbol{Y}^T - \boldsymbol{M} \|_F^2,
\end{equation}
where $\boldsymbol{M} = \boldsymbol{U}_*\boldsymbol{\Sigma}\boldsymbol{V}_*^T$ is its Singular Value Decomposition (SVD), with $\boldsymbol{U}_* \in \mathbb{R}^{d_1 \times r}$ and $\boldsymbol{V}_* \in \mathbb{R}^{d_2 \times r}$ having orthonormal columns.

\paragraph{The Reduced Problem on the Grassmannian}
For any given full-rank matrix $\boldsymbol{X}$, the optimal $\boldsymbol{Y}$ is given by $\boldsymbol{Y}^*(\boldsymbol{X}) = \boldsymbol{M}^T\boldsymbol{X}(\boldsymbol{X}^T\boldsymbol{X})^{-1}$. Substituting this back into the objective yields the reduced function:
\begin{equation}
    f(\boldsymbol{X}) = \frac{1}{2}\| (\boldsymbol{I} - \boldsymbol{P_X})\boldsymbol{M} \|_F^2,
\end{equation}
where $\boldsymbol{P_X} = \boldsymbol{X}(\boldsymbol{X}^T\boldsymbol{X})^{-1}\boldsymbol{X}^T$ is the orthogonal projection operator onto the column space of $\boldsymbol{X}$, denoted $\mathcal{C}(\boldsymbol{X})$.

The crucial observation is that the value of $f(\boldsymbol{X})$ depends only on the subspace $\mathcal{C}(\boldsymbol{X})$, not on the specific basis $\boldsymbol{X}$ that spans it. That is, for any invertible matrix $\boldsymbol{R} \in \mathbb{R}^{r \times r}$, we have $f(\boldsymbol{X}) = f(\boldsymbol{XR})$. Consequently, $f$ is naturally defined on the \textbf{Grassmannian manifold}, $Gr(r, d_1)$, which is the space of all $r$-dimensional linear subspaces of $\mathbb{R}^{d_1}$. Letting $\mathcal{S} \in Gr(r, d_1)$ denote such a subspace, the reduced problem can be elegantly expressed as:
\begin{equation}
    \min_{\mathcal{S} \in Gr(r, d_1)} f(\mathcal{S}) = \frac{1}{2}\| (\boldsymbol{I} - \boldsymbol{P}_{\mathcal{S}})\boldsymbol{M} \|_F^2,
\end{equation}
where $\boldsymbol{P}_{\mathcal{S}}$ is the orthogonal projector onto the subspace $\mathcal{S}$.

\paragraph{Mapping of Saddle Points under the Landscape Transformation}
In the landscape of the original problem $F(\boldsymbol{X}, \boldsymbol{Y})$, an important class of strict saddle points consists of points of the form $(\boldsymbol{X}_s, \boldsymbol{0})$. These saddle points are characterized by $\boldsymbol{X}_s$ being full-rank, yet its column space $\mathcal{C}(\boldsymbol{X}_s)$ is orthogonal to the true signal subspace of $\boldsymbol{M}$, which is $\mathcal{C}(\boldsymbol{U}_*)$. Formally, this condition is written as $\mathcal{C}(\boldsymbol{X}_s) \perp \mathcal{C}(\boldsymbol{U}_*)$, or more compactly, $\boldsymbol{U}_*^T\boldsymbol{X}_s = \boldsymbol{0}$.

We now demonstrate that under the variable projection mapping, the subspaces $\mathcal{S}_s = \mathcal{C}(\boldsymbol{X}_s)$ corresponding to these saddle points become the global maximizers of the reduced function $f(\mathcal{S})$ on the Grassmannian.

\begin{proof}
The reduced objective function $f(\mathcal{S})$ can be decomposed into a constant and a term dependent on the subspace $\mathcal{S}$:
$$f(\mathcal{S}) = \frac{1}{2} \left( \|\boldsymbol{M}\|_F^2 - \|\boldsymbol{P}_{\mathcal{S}}\boldsymbol{M}\|_F^2 \right).$$
Using the SVD of $\boldsymbol{M}$ and the property that $\boldsymbol{V}_*^T\boldsymbol{V}_* = \boldsymbol{I}_r$, we analyze the second term:
$$\|\boldsymbol{P}_{\mathcal{S}}\boldsymbol{M}\|_F^2 = \|\boldsymbol{P}_{\mathcal{S}}\boldsymbol{U}_*\boldsymbol{\Sigma}\boldsymbol{V}_*^T\|_F^2 = \|\boldsymbol{P}_{\mathcal{S}}\boldsymbol{U}_*\boldsymbol{\Sigma}\|_F^2 = \text{Tr}(\boldsymbol{\Sigma}^T \boldsymbol{U}_*^T \boldsymbol{P}_{\mathcal{S}} \boldsymbol{U}_* \boldsymbol{\Sigma}).$$
This term represents the energy of the signal component of $\boldsymbol{M}$ projected onto the subspace $\mathcal{S}$. Its magnitude depends on the alignment between $\mathcal{S}$ and the true signal subspace $\mathcal{C}(\boldsymbol{U}_*)$, a relationship that can be quantified by the \textbf{principal angles} between the two subspaces.

Now, consider the subspace $\mathcal{S}_s = \mathcal{C}(\boldsymbol{X}_s)$ corresponding to a saddle point. By definition, this subspace is orthogonal to $\mathcal{C}(\boldsymbol{U}_*)$. This implies that the projection of any vector from $\mathcal{C}(\boldsymbol{U}_*)$ onto $\mathcal{S}_s$ is the zero vector. Consequently, the action of the projection operator is:
$$\boldsymbol{P}_{\mathcal{S}_s}\boldsymbol{U}_* = \boldsymbol{0}.$$
Substituting this result into the expression for $f(\mathcal{S}_s)$ yields:
$$f(\mathcal{S}_s) = \frac{1}{2} \left( \|\boldsymbol{M}\|_F^2 - \|\boldsymbol{0} \cdot \boldsymbol{\Sigma}\|_F^2 \right) = \frac{1}{2}\|\boldsymbol{M}\|_F^2 = \frac{1}{2}\sum_{i=1}^r \sigma_i^2.$$
For any arbitrary subspace $\mathcal{S}$ on the Grassmannian, the squared Frobenius norm $\|\boldsymbol{P}_{\mathcal{S}}\boldsymbol{M}\|_F^2$ is non-negative. Therefore,
$$f(\mathcal{S}) = \frac{1}{2} \left( \|\boldsymbol{M}\|_F^2 - \|\boldsymbol{P}_{\mathcal{S}}\boldsymbol{M}\|_F^2 \right) \le \frac{1}{2}\|\boldsymbol{M}\|_F^2.$$
This inequality shows that $f(\mathcal{S}_s)$ is the global maximum value of the reduced function.
\end{proof}

\begin{remark}
    This analysis reveals a profound geometric transformation. The set of strict saddle points in the ambient Euclidean space, defined by rank deficiency and orthogonality to the true signal subspace, is precisely mapped to the set of global maximizers on the Grassmannian manifold via variable projection. This fundamentally clarifies why this reduction method effectively circumvents these particular saddles: an optimization algorithm on the reduced landscape is naturally repelled from these "peaks" and is instead guided towards the "valleys" (the global minima) that contain the true signal directions. This provides a perfect, high-dimensional illustration of Theorem 5, revealing that the power of variable elimination lies in its ability to fundamentally re-sculpt the optimization landscape, transforming treacherous saddle points into easily avoidable maxima. 
\end{remark} 

\section{Numerical Experiments}
This section provides an intuitive visualization and empirical validation of the theoretical findings presented earlier. We first use simple, two-parameter networks to explicitly visualize the cost function landscape transformation, and then demonstrate the practical consequences of this transformation through a statistical comparison of algorithm performance. 

\subsection{Visualizing the Transformed Landscape}
To visualize the geometry of the full cost function versus that of the reduced cost function, we employ simple two-parameter multi-layer perception (MLP) and radial-basis-function (RBF) neural networks as a "mathematical microscope". The two models, adopted from \cite{mcloone:02}, are:
\begin{enumerate}
    \item \textbf{MLP with Sigmoid Activation:} The network output is given by $y = w_L \cdot \sigma(w_N x)$, where is the logistic sigmoid function.
    \item \textbf{RBF Network with Gaussian Activation:} The output is given by $y = w_L \cdot \exp(-(x - w_N)^2)$. Here, $w_N$ acts as the center of the Gaussian kernel.
\end{enumerate}

In both models, $w_L$ and $w_N$ are the linear and nonlinear parameters, respectively. A small dataset\footnote{The training data consisted of inputs $X = [-0.5  -0.2  0.0  0.3  0.8  0.5  -0.11]$ and corresponding targets $Y = [-0.02 -0.03 -0.01 -0.03 0.02 0.02 0.05]$ are used for training MLP and $X = [1.4  -2.1  0.7  0.5  1.2 -1.5 1 2.1 -0.3 -0.6 -1 2.12]$, $Y = [-0.2 -0.6 0.2 0.2 -1.5 0.1 -0.5 -1 1 2 3 -2.1]$ for RBF.} was used to generate the cost functions. For each model, we visualize and compare two cost functions based on the Squared Error (SE). The resulting landscapes are presented in Figure \ref{fig:sigmoid_landscape}.

\begin{figure}[!t]
    \centering
    \includegraphics[width=1.1\textwidth]{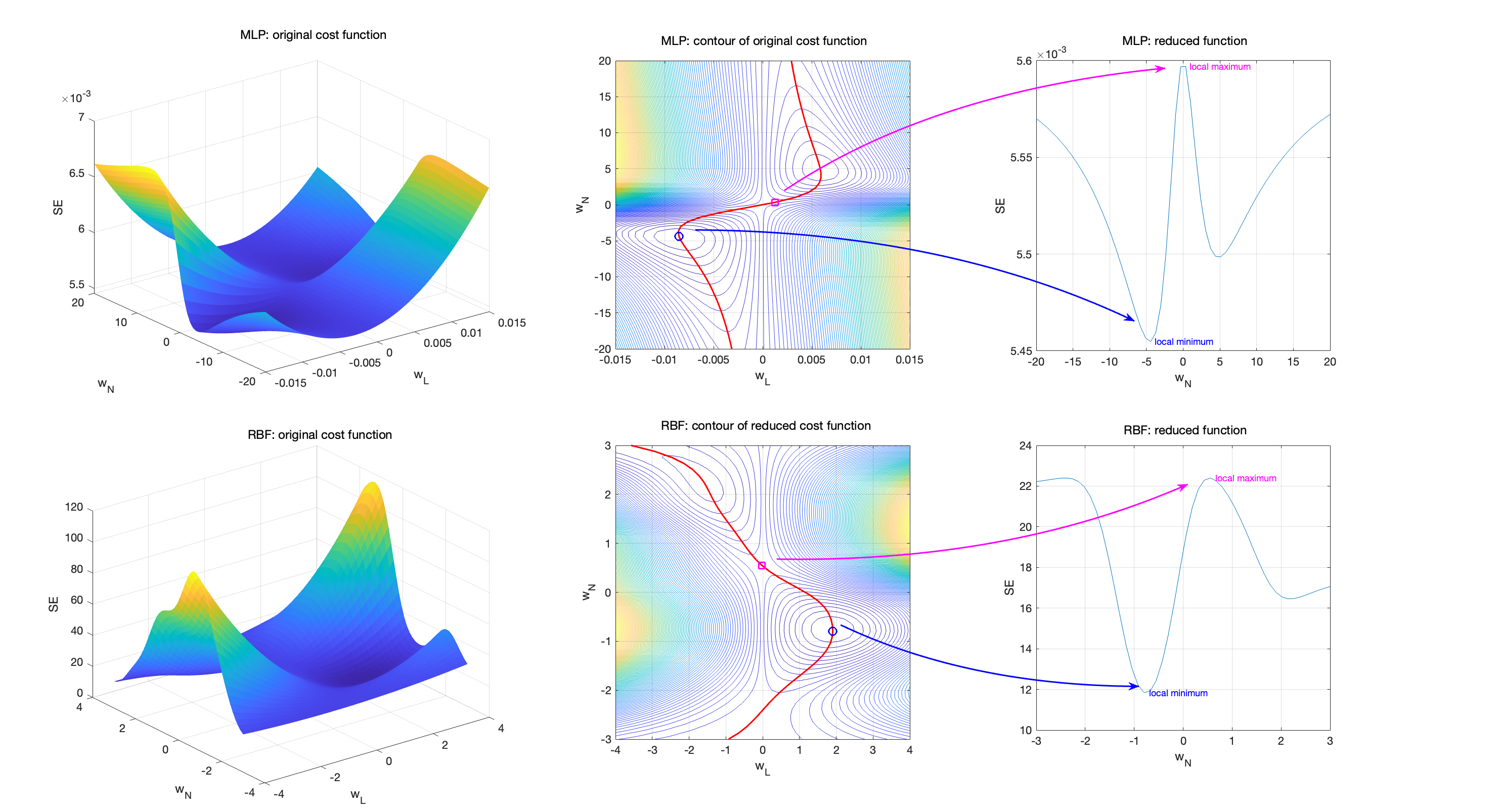} 
    \caption{Visualization of the original cost function and the reduced cost function for the MLP (top row) and RBF (bottom row) models. \textbf{Left:} The 3D surface of the original cost function. \textbf{Middle:} The contour plot of the original cost function, where the red curve represents the manifold of optimal linear parameters. \textbf{Right:} The corresponding one-dimensional reduced cost function. Arrows explicitly illustrate the mapping of stationary points.}
    \label{fig:sigmoid_landscape}
\end{figure}

Figure \ref{fig:sigmoid_landscape} provides a direct, visual confirmation of our theoretical framework. The transformation is particularly striking for the MLP model (top row), where the complex, winding valley of the original landscape is simplified into a well-behaved 1D function. Most importantly, the arrows explicitly show a saddle point in the original space being mapped to a simple local maximum in the reduced space---a textbook illustration of Theorem 5. This inertia-preserving transformation offers a clear, geometric explanation for the robust performance of variable elimination algorithms, showing how they simplify the landscape by converting challenging saddle points into easily avoidable maxima.

\subsection{Training of feed-forward neural networks}
The training of neural networks is a canonical example of high-dimensional, non-convex optimization. To demonstrate the practical impact of variable elimination methods, we apply the joint method and VarPro to the training of a single-hidden-layer feed-forward neural network with squared loss function. Both approaches utilize the Levenberg-Marquardt (LM) algorithm as the core optimizer, applied to the original and reduced objective functions, respectively. 

The experimental setup follows the "teacher-student" paradigm. A "teacher" network with 5 hidden neurons using the logistic sigmoid activation function and a linear output layer was created with known, randomly generated weights. This teacher network was then used to generate a dataset consisting of 300 input-output pairs. Subsequently, a "student" network of identical architecture was trained by each algorithm, starting from a different random initialization for each trial. The goal is to evaluate how effectively each method can solve the optimization problem and recover the underlying teacher model.

To account for the high sensitivity of non-convex optimization to initial conditions, we conducted 100 training runs for each algorithm, each starting from a different, randomly generated set of initial parameters. The distribution of the final residual errors from these 100 runs is summarized in the box plot shown in Figure \ref{fig:boxplot_comparison}. The performance metric is the logarithm of the final Residual Sum of Squares (lg RSS); lower values signify a better fit to the data.

\begin{figure}[!t]
    \centering
    \includegraphics[width=0.6\textwidth]{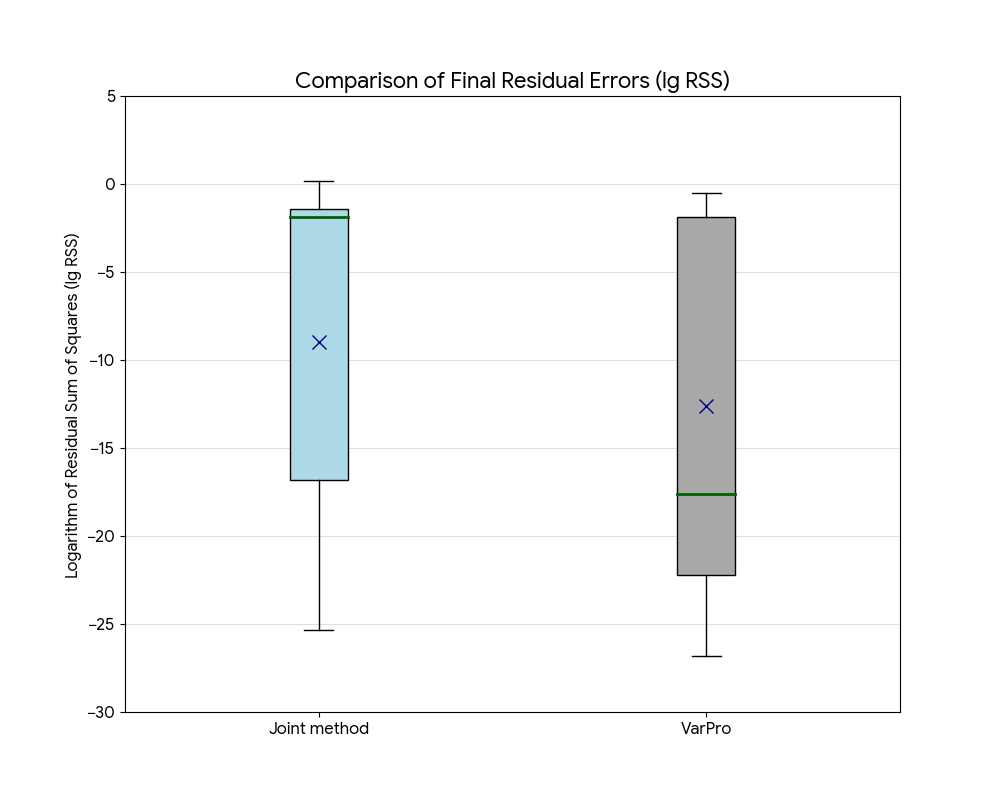} 
    \caption{A box plot comparing the distribution of the final residual sum of squares (RSS) for the joint Levenberg-Marquardt (Joint LM) algorithm and the Variable Projection (VarPro) algorithm after 100 training runs from random initializations. The VarPro method demonstrates significantly lower median error and superior consistency.}
    \label{fig:boxplot_comparison}
\end{figure}

The results in Figure \ref{fig:boxplot_comparison} reveal a dramatic performance gap between the two methods. This aligns with insights from modern optimization theory, which posits that in high-dimensional non-convex problems, the primary difficulty is not poor local minima, but the proliferation of saddle points \citep{dauphin:14}. It is argued that critical points with errors significantly higher than the global minimum are exponentially likely to be saddle points \citep{choromanska:15}, whereas true local minima tend to have errors very close to that of the global optimum.

Our experimental results illustrate  this perspective:

\begin{itemize}
    \item \textbf{The Failure Mode of Joint Optimization}: The performance of the joint optimization method is highly erratic. A large portion of its results, indicated by the high median and wide variance, terminate at very high error values (lg RSS $>$ -2). According to the theory \citep{choromanska:15}, these are not high-cost local minima but are overwhelmingly likely to be regions where the optimizer has been stalled by the complex plateaus surrounding saddle points.
    
    \item \textbf{The Success of VarPro}: In stark contrast, the VarPro algorithm consistently converges to solutions with very low residual error (median lg RSS near -20). These outcomes correspond to high-quality solutions that, by the same logic, reside in the basin of attraction of the true \textbf{global or near-global minima}.
\end{itemize}

This experiment provides strong empirical evidence for our central thesis. The standard joint optimization method is highly susceptible to the saddle point problem that plagues neural network training. Conversely, the Variable Projection method---by fundamentally simplifying the optimization landscape as proven in our theoretical analysis (Theorem 5)---effectively circumvents these saddle point traps. This reshaping of the cost surface enables the optimizer to reliably find paths to desirable, low-error minima, explaining its superior performance and robustness.

\subsection{Training of Residual Networks}
Having demonstrated the principles of landscape simplification on shallow networks, a critical question remains: do these benefits extend to the complex, high-dimensional problems in modern deep learning? To address this, we investigate the training of a deep Residual Network (ResNet), a cornerstone of modern architectures.

\subsubsection{Experimental Setup}

The experiment compares the performance of two training strategies on a function approximation task, where the goal is to fit the function $y = \sin(2\pi x_1)*\cos(2\pi x_2)$.

\begin{itemize}
    \item \textbf{Model}: We employ a deep ResNet composed of an input layer, a sequence of \textbf{8 residual blocks}, and a final linear output layer. Each hidden layer within the architecture has a width of 64 units and uses the `tanh` activation function.
    
    \item \textbf{Algorithms}: We compare standard end-to-end \textbf{Gradient Descent (GD)} against a hybrid method embodying the Variable Projection principle, termed \textbf{Least-Squares Gradient Descent (LSGD)} \citep{cyr:20}. The GD method jointly optimizes all network parameters using Adam \citep{kingma:15}, while the LSGD method iteratively updates the nonlinear parameters (all hidden layers) with Adam and analytically solves for the optimal linear parameters (the final layer) in each step.

    \item \textbf{Methodology}: We conduct 16 independent trials for each algorithm from different random initializations, each for 10,000 epochs. Learning rates were tuned independently for each method to ensure a fair comparison.
\end{itemize}

\subsubsection{Results and Analysis}
The training dynamics, averaged over the 16 trials, are presented in Figure \ref{fig:resnet_training}. The plot shows the mean training loss (on a log10 scale) as a function of the training step (on a log scale). The shaded regions represent the standard deviation across the trials, providing a clear measure of each algorithm's stability and robustness to initialization.

\begin{figure}[!t]
    \centering
    \includegraphics[width=0.6\textwidth]{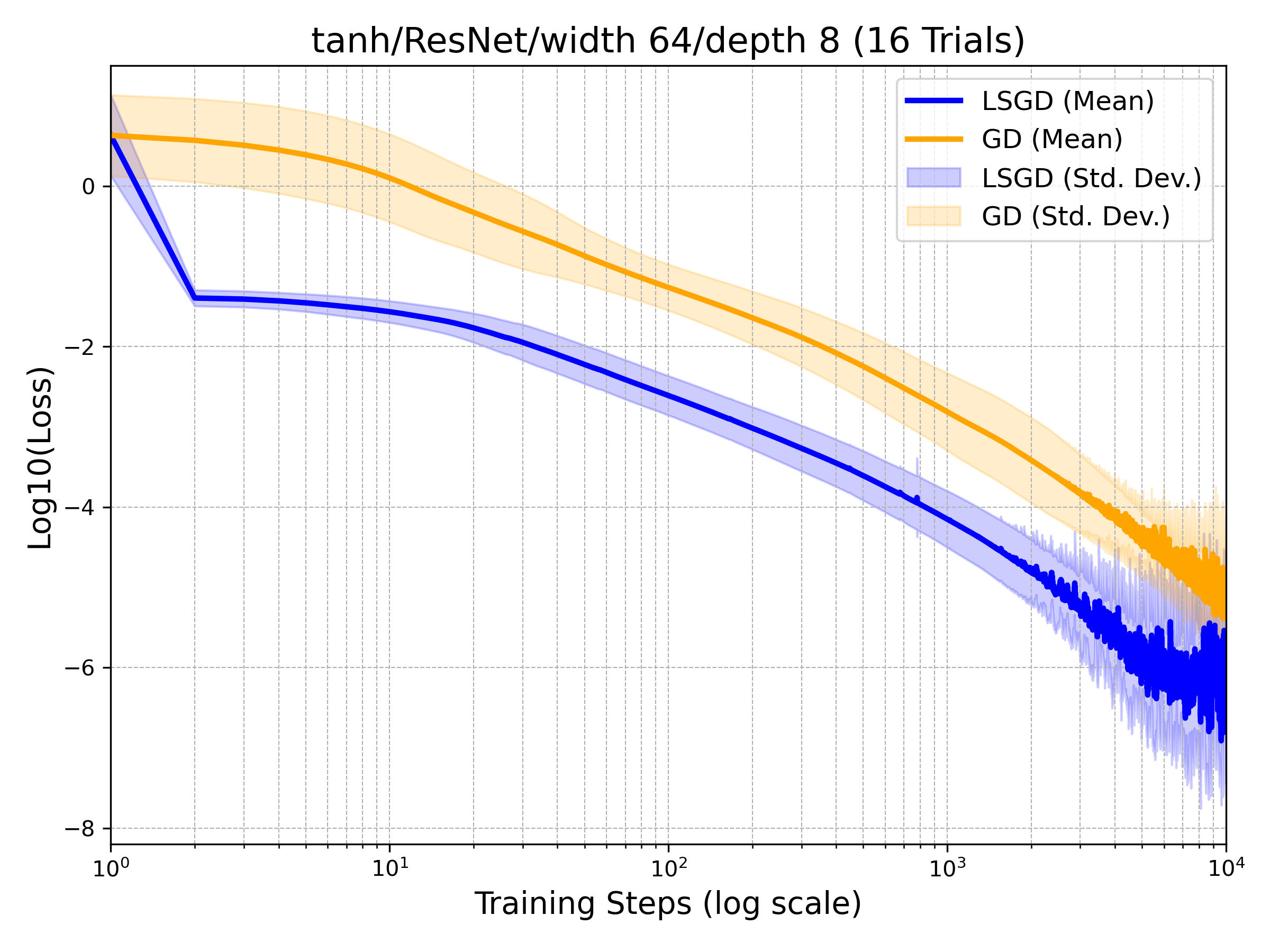} 
    \caption{Comparison of training dynamics for standard Gradient Descent (GD) and the Least-Squares Gradient Descent (LSGD/VarPro) on a deep ResNet. The solid lines represent the mean loss over 16 trials; shaded areas are the standard deviation.}
    \label{fig:resnet_training}
\end{figure}

The results are conclusive and demonstrate that the advantages of variable elimination extend to deep learning models.

\begin{enumerate}
    \item \textbf{Convergence Speed and Final Error}: The LSGD (VarPro) method exhibits a dramatic initial drop in loss due to the optimal least-squares step on the output weights. It converges significantly faster and achieves a lower final mean error than standard GD.

    \item \textbf{Robustness and Stability}: Most strikingly, the shaded region for the standard GD algorithm is notably wide, implying that its performance is sensitive to the initial parameters. Compared to the wide variance of GD, the standard deviation for LSGD is markedly smaller, indicating that it converges to a similar high-quality solution with much greater consistency from different starting points.
\end{enumerate}

These results provide a compelling demonstration of our central thesis in a practical deep learning setting. The instability and slower convergence of standard gradient descent are classic symptoms of an optimizer struggling in a high-dimensional landscape. In contrast, the exceptional speed and robustness of the LSGD approach are powerful empirical evidence of the underlying landscape simplification predicted by our theory. By eliminating the linear parameters, VarPro does more than accelerate training; it fundamentally reshapes the cost surface. This allows the optimizer to consistently and efficiently find high-quality minima, showcasing the immense potential of variable elimination as a principle for designing more robust and effective training algorithms for modern deep networks.

\section{Conclusion}

This paper has sought to provide a deep and principled explanation for the widely observed, yet not fully understood, success of variable elimination algorithms in machine learning. Our work moves beyond simple observations of faster convergence and instead investigates the fundamental mechanism by which these methods achieve superior performance and robustness. The central thesis of our research is that the efficacy of the Variable Elimination approach stems from its ability to fundamentally reshape and simplify the underlying optimization landscape.

Our theoretical analysis, grounded in the properties of Schur complements and Hessian inertia, has rigorously established this claim. We first proved a direct correspondence between the stationary points of the original problem and the reduced-dimension problem. The most significant theoretical contribution is the discovery, via the Haynsworth inertia additivity formula, that the number of negative-curvature directions is perfectly conserved during the transformation. This principle leads to a crucial consequence: a specific class of saddle points in the original landscape--those that are maximal in the primary directions and minimal in the eliminated directions--are systematically converted into local maxima in the reduced landscape.

These theoretical insights were validated across a hierarchy of numerical experiments. In low-dimensional networks, we explicitly visualized a saddle point transforming into a local maximum. This phenomenon's practical impact was confirmed in statistical trials on shallow networks and scaled up to modern deep Residual Networks, where the VarPro-based method demonstrated dramatically superior convergence speed and robustness compared to the standard end-to-end approach.

In conclusion, this work provides a clear and compelling answer to why VarPro is so effective: it solves an easier problem. By analytically optimizing a subset of parameters, it does not merely reduce the dimensionality of the search space but also fundamentally improves its geometry. By transforming "trapping" saddle points into "repelling" local maxima, the algorithm can more effectively navigate the error surface to find superior solutions.

Future work could extend this analysis to more complex, deep learning architectures. Investigating whether similar landscape simplifications occur in networks with multiple hidden layers or different activation functions remains an open and promising area of research. Additionally, exploring how these geometric properties affect the behavior of stochastic, gradient-based optimizers could yield valuable insights for developing more robust training algorithms for the next generation of neural networks.

\acks{This work was supported by the National Nature Science Foundation of China under Grant 62576183 and Grant 62173091.}


\newpage

\appendix
\setcounter{equation}{0}
\numberwithin{equation}{section}
\renewcommand{\theequation}{\thesection\arabic{equation}}
\section{}
  Consider the separable nonlinear least squares problems with an objective function
\begin{equation}
    \|\boldsymbol\epsilon(\boldsymbol{x}, \boldsymbol{y})\|_2^2 = \frac{1}{2}\|\boldsymbol G(\boldsymbol{x})\boldsymbol{y} - \boldsymbol z\|_2^2
\end{equation}
We denote by $\boldsymbol{J}_{\boldsymbol{x}}(\boldsymbol{x}, \boldsymbol{y})=\partial \boldsymbol\epsilon(\boldsymbol{x}, \boldsymbol{y})/\partial \boldsymbol{x}$ and $\boldsymbol{J}_{\boldsymbol{y}}(\boldsymbol{x})=\partial \boldsymbol\epsilon(\boldsymbol{x}, \boldsymbol{y})/{\partial \boldsymbol{y}}$  the Jacobian matrices with respect to $\boldsymbol{x}$ and $\boldsymbol{y}$, respectively. Note that 
 $\boldsymbol{J}_{\boldsymbol{y}}(\boldsymbol{x})$ depends only on $\boldsymbol{x}$, since $\boldsymbol\epsilon(\boldsymbol{x}, \boldsymbol{y})$ is linear in $\boldsymbol{y}$. 
 
 To derive update rules, we can linearize the residual around the current iterate $(\boldsymbol{x}_k,\boldsymbol{y}_k)$. Using a first-order Taylor expansion (essentially a Gauss-Newton approximation), the residual can be approximated as:
\begin{equation}
	\boldsymbol\epsilon(\boldsymbol{x}_k+\Delta \boldsymbol{x},\boldsymbol{y}_k + \Delta \boldsymbol{y}) \approx 
	\boldsymbol\epsilon(\boldsymbol{x}_k,\boldsymbol{y}_k ) + \boldsymbol{J}_{\boldsymbol{x}}(\boldsymbol{x}_k,\boldsymbol{y}_k)\Delta \boldsymbol{x} + \boldsymbol{J}_{\boldsymbol{y}}(\boldsymbol{x}_k) \Delta \boldsymbol{y}.
\end{equation}

If we fix $\Delta \boldsymbol{x}$ (i.e. assume a tentative update in the nonlinear parameters), the optimal adjustment $\Delta \boldsymbol{y}$ can be obtained by minimizing
\begin{equation}
    \| \boldsymbol\epsilon(\boldsymbol{x}_k,\boldsymbol{y}_k ) + \boldsymbol{J}_{\boldsymbol{x}}(\boldsymbol{x}_k,\boldsymbol{y}_k)\Delta \boldsymbol{x} + \boldsymbol{J}_{\boldsymbol{y}}(\boldsymbol{x}_k) \Delta \boldsymbol{y} \| .
\end{equation}
The solution is $\Delta \boldsymbol{y} = - \boldsymbol{J}_{\boldsymbol{y}}(\boldsymbol{x}_k)^{\dagger}(\boldsymbol\epsilon(\boldsymbol{x}_k,\boldsymbol{y}_k )+\boldsymbol{J}_{\boldsymbol{x}}(\boldsymbol{x}_k,\boldsymbol{y}_k )\Delta \boldsymbol{x})$. 
This is what Hong et al. derived in \cite{hong:17}. It is not an exact update for $\boldsymbol{y}$, but rather an approximate analytical expression based on a first-order linearization.

In the actual implementation of VarPro, we do not explicitly compute $\boldsymbol{y}$ because it has already been eliminated. However, after obtaining $\boldsymbol{x}_{k+1}$, there will be an implicitly corresponding $\boldsymbol{y}_{k+1}$. Because the $\boldsymbol{y}$-subproblem 
\begin{equation}
    \|\boldsymbol\epsilon(\boldsymbol{x}_{k+1}, \boldsymbol{y})\|_2^2 = \frac{1}{2}\|\boldsymbol G(\boldsymbol{x}_{k+1})\boldsymbol{y} - \boldsymbol z\|_2^2
\end{equation}
is linear, a single Newton/least-squares step finds the optimum in one shot. Thus, an exact update from $\boldsymbol{y}_k$ to $\boldsymbol{y}_{k+1}$ is 
\begin{equation}
   \Delta \boldsymbol{y} = - \boldsymbol{J}_{\boldsymbol{y}}(\boldsymbol{x}_{k+1})^{\dagger}\boldsymbol\epsilon(\boldsymbol{x}_{k+1},\boldsymbol{y}_k ).
\end{equation}

\section{}
\textbf{Scenario:} The $\boldsymbol{y}$-component of the original critical point is not the minimizer of the inner problem. This occurs when $\boldsymbol{y}_0$ is a stationary point of $F(\boldsymbol{x}_0,\cdot)$, but not the point that defines $\boldsymbol{y}^*(\boldsymbol{x}_0)$. 

Here we provide an example to illustrate how a critical point $(\boldsymbol{x}_c, \boldsymbol{y}_c)$ of the original problem is "filtered out" by the variable elimination process if it does not satisfy the condition $\boldsymbol{y}_c = \boldsymbol{y}^{*}(\boldsymbol{x}_c)$.
Consider the function:
\begin{equation}
    F(x, y) = xy + \frac{1}{4}y^4 - y^3
\end{equation}
\textit{Note: To construct an example where a critical point exists off the optimal submanifold (i.e., $\boldsymbol{y}_c \neq \boldsymbol{y}^{*}(\boldsymbol{x}_c)$), the inner problem with respect to $\boldsymbol{y}$ must have at least two stationary points. Consequently, it must be at least a cubic polynomial in $\boldsymbol{y}$ and cannot be a simple quadratic, which is the typical tractable case for VarPro. This example is chosen for its theoretical clarity in demonstrating the filtering mechanism.}

We proceed with a step-by-step analysis:

\textbf{1. Find the critical points of $F(x,y)$.}
\begin{itemize}
    \item $\nabla_x F(x,y) = y = 0$
    \item $\nabla_y F(x,y) = x + y^3 - 3y^2 = 0$
\end{itemize}
Therefore, the function has a unique critical point at $(x_c, y_c) = (0,0)$.

\textbf{2. Analyze the inner problem at $x_c = 0$.}

We seek to solve $\min_y F(0, y)$. The objective for the inner problem is:
$$ F(0, y) = \frac{1}{4}y^4 - y^3 $$
To find its minimizer, we compute the derivative with respect to $y$:
$$ \nabla_y F(0, y) = y^3 - 3y^2 = y^2(y - 3) $$
Setting the derivative to zero gives two stationary points for the inner problem: $y=0$ and $y=3$. By analyzing the function's behavior (it tends to $+\infty$ as $|y| \to \infty$), we find that $y=3$ is the **unique global minimizer**. Thus, we have $\boldsymbol{y}^{*}(0) = 3$.

\textbf{3. Verify the key condition.}

We compare the y-component of the critical point with the solution to the inner problem:
\begin{itemize}
    \item The original critical point is $(x_c, y_c) = (0,0)$.
    \item The unique optimal solution to the inner problem at $x_c=0$ is $y^*(0) = 3$.
\end{itemize}
Clearly, $y_c \neq y^*(x_c)$, as $0 \neq 3$. The critical point $(0,0)$ does not lie on the optimal submanifold $\mathcal{M}$.

\textbf{4. Check the gradient of the reduced function.}

The gradient of the reduced function $\tilde{F}(x)$ is given by $\nabla\tilde{F}(x) = \nabla_x F(x, y^*(x))$. We evaluate this at the point of interest, $x_c=0$:
$$ \nabla\tilde{F}(0) = \nabla_x F(0, y^*(0)) = \nabla_x F(0, 3) $$
Since $\nabla_x F(x,y) = y$, we have:
$$ \nabla\tilde{F}(0) = y \Big|_{y=3} = 3.$$
Because $\nabla\tilde{F}(0) = 3 \neq 0$, the point $x_c=0$ is \textbf{not a critical point} of the reduced problem $\tilde{F}(x)$.

\textbf{Further Insight:}

It is noteworthy that in this example, the point $(0,0)$ is the \textit{only} critical point of the original problem $F(x,y)$. Our analysis shows that this single critical point—which is a saddle point—is filtered out by the variable elimination process, resulting in a reduced problem $\tilde{F}(x)$ that has **no critical points at all**. This is not a failure of the method. On the contrary, the original problem $F(x,y)$ is unbounded and has no minima. The reduced problem correctly reflects this by presenting a landscape with no stationary points (a monotonic function in this case), effectively signaling to any optimization algorithm that no local minimum exists. This provides a stark and powerful illustration of how variable elimination simplifies the landscape by removing even unique, yet undesirable, stationary points.

\section{}

\begin{theorem}
\label{thm:minima_correspondence}
Let $F: \mathbb{R}^p \times \mathbb{R}^q \to \mathbb{R}$ be a $C^2$ function, and assume that for each $\boldsymbol{x} \in \mathbb{R}^p$, the inner problem $\min_{\boldsymbol{y} \in \mathbb{R}^q} F(\boldsymbol{x},\boldsymbol{y})$ has a unique solution $\boldsymbol{y}^{*}(\boldsymbol{x})$, and that $\boldsymbol{y}^{*}(\boldsymbol{x})$ is a $C^1$ function. Furthermore, assume that the Hessian matrix of the inner problem, $\nabla_{\boldsymbol{y}\boldsymbol{y}}^2 F(\boldsymbol{x},\cdot)$, is positive definite.

(i) If $(\boldsymbol{x}_0, \boldsymbol{y}_0)$ is a local minimum of the original problem $F(\boldsymbol{x}, \boldsymbol{y})$ where $\boldsymbol{y}_0 = \boldsymbol{y}^*(\boldsymbol{x}_0)$, then $\boldsymbol{x}_0$ is a local minimum of the reduced problem $\tilde{F}(\boldsymbol{x})$.

(ii) If $\boldsymbol{x}_0$ is a local minimum of the reduced problem $\tilde{F}(\boldsymbol{x})$, then $(\boldsymbol{x}_0, \boldsymbol{y}_0)$ is a local minimum of the original problem $F(\boldsymbol{x}, \boldsymbol{y})$, where $\boldsymbol{y}_0 = \boldsymbol{y}^*(\boldsymbol{x}_0)$.
\end{theorem}

\begin{proof}

\textbf{Proof of (i) (Original Problem $\Rightarrow$ Reduced Problem):}

The variable elimination method essentially constrains the optimization of $F(\boldsymbol{x}, \boldsymbol{y})$ to a $p$-dimensional submanifold $\mathcal{M} \subset \mathbb{R}^{p+q}$ defined by the mapping $\boldsymbol{y} = \boldsymbol{y}^*(\boldsymbol{x})$. The reduced function $\tilde{F}(\boldsymbol{x}) = F(\boldsymbol{x}, \boldsymbol{y}^*(\boldsymbol{x}))$ is precisely the original function $F$ restricted to this submanifold $\mathcal{M}$.

By the definition of a local minimum, if $(\boldsymbol{x}_0, \boldsymbol{y}_0)$ is a local minimum of $F$, then there exists a neighborhood $\mathcal{N} \subset \mathbb{R}^{p+q}$ around $(\boldsymbol{x}_0, \boldsymbol{y}_0)$ such that for all $(\boldsymbol{x}, \boldsymbol{y}) \in \mathcal{N}$:
\[
F(\boldsymbol{x}_0, \boldsymbol{y}_0) \le F(\boldsymbol{x}, \boldsymbol{y}).
\]
Since $\boldsymbol{y}_0 = \boldsymbol{y}^*(\boldsymbol{x}_0)$, the point $(\boldsymbol{x}_0, \boldsymbol{y}^*(\boldsymbol{x}_0))$ lies on the submanifold $\mathcal{M}$. The set of points on the submanifold $\mathcal{M}$ within a sufficiently small neighborhood of $(\boldsymbol{x}_0, \boldsymbol{y}_0)$ is a subset of the points in the full-space neighborhood $\mathcal{N}$. Therefore, the inequality must also hold for this subset. This implies that for all $\boldsymbol{x}$ in a neighborhood of $\boldsymbol{x}_0$ (such that $(\boldsymbol{x}, \boldsymbol{y}^*(\boldsymbol{x})) \in \mathcal{N}$), we have:
\[
F(\boldsymbol{x}_0, \boldsymbol{y}^*(\boldsymbol{x}_0)) \le F(\boldsymbol{x}, \boldsymbol{y}^*(\boldsymbol{x})).
\]
By the definition of the reduced function, this is equivalent to $\tilde{F}(\boldsymbol{x}_0) \le \tilde{F}(\boldsymbol{x})$. This is precisely the definition of $\boldsymbol{x}_0$ as a local minimum of $\tilde{F}(\boldsymbol{x})$.

\textbf{Proof of (ii) (Reduced Problem $\Rightarrow$ Original Problem):}
This proof is based directly on the definition of a local minimum and holds true even in the degenerate case where the Hessians are only positive semi-definite.

    Since $\boldsymbol{x}^*$ is a local minimum of $\tilde{F}(\boldsymbol{x})$, by definition, there exists a neighborhood $\mathcal{N}_x$ around $\boldsymbol{x}^*$ such that for all $\boldsymbol{x} \in \mathcal{N}_x$, the following inequality holds:
    \[
    \tilde{F}(\boldsymbol{x}^*) \le \tilde{F}(\boldsymbol{x})
    \]
    Substituting the definition of $\tilde{F}(\boldsymbol{x})$ yields:
    \begin{equation} \label{eq:reduced_min}
    F(\boldsymbol{x}^*, \boldsymbol{y}^*(\boldsymbol{x}^*)) \le F(\boldsymbol{x}, \boldsymbol{y}^*(\boldsymbol{x}))
    \end{equation}

    By the definition of variable elimination, $\boldsymbol{y}^*(\boldsymbol{x})$ is the solution to the inner problem $\min_{\boldsymbol{y}} F(\boldsymbol{x}, \boldsymbol{y})$. This means that for any given $\boldsymbol{x}$ and for any choice of $\boldsymbol{y}$, the following inequality always holds:
    \begin{equation} \label{eq:inner_min}
    F(\boldsymbol{x}, \boldsymbol{y}^*(\boldsymbol{x})) \le F(\boldsymbol{x}, \boldsymbol{y})
    \end{equation}

    Now, we connect these two inequalities. For any point $(\boldsymbol{x}, \boldsymbol{y})$ in a neighborhood of $(\boldsymbol{x}^*, \boldsymbol{y}^*(\boldsymbol{x}^*))$ (where $\boldsymbol{x} \in \mathcal{N}_x$), we can form a chain using \eqref{eq:reduced_min} and \eqref{eq:inner_min}:
    \[
    F(\boldsymbol{x}^*, \boldsymbol{y}^*(\boldsymbol{x}^*)) \stackrel{\text{by \eqref{eq:reduced_min}}}{\le} F(\boldsymbol{x}, \boldsymbol{y}^*(\boldsymbol{x})) \stackrel{\text{by \eqref{eq:inner_min}}}{\le} F(\boldsymbol{x}, \boldsymbol{y})
    \]
    This chain of inequalities directly gives us the result:
    \[
    F(\boldsymbol{x}^*, \boldsymbol{y}^*(\boldsymbol{x}^*)) \le F(\boldsymbol{x}, \boldsymbol{y})
    \]
    This inequality shows that within a neighborhood of $(\boldsymbol{x}^*, \boldsymbol{y}^*(\boldsymbol{x}^*))$, the function value at this point is a minimum. This is precisely the definition of a local minimum for the original problem $F(\boldsymbol{x}, \boldsymbol{y})$.

\end{proof}

\vskip 0.2in
\bibliography{sample}

\end{document}